\documentclass{article}
\pdfoutput=1
\usepackage{silence}
\usepackage[preprint, nonatbib]{Style/neurips_2022}
\usepackage{import}
\usepackage{amsmath,etoolbox}
\usepackage{commath}
\usepackage{graphicx}
\usepackage{amsthm}
\usepackage{amssymb}
\usepackage{bbm}
\usepackage{amsfonts}
\usepackage{enumitem} 
\usepackage{algorithm}
\usepackage{algorithmicx, setspace}
\usepackage[noend]{algpseudocode}

\newcommand{\Ex}{\mathbb{E}}

\newcommand{\N}{\mathbb{N}}

\renewcommand{\P}{\mathbb{P}}

\newcommand{\R}{\mathbb{R}}

\newcommand{\FF}{\mathcal{F}}

\newcommand{\LL}{\mathcal{L}}

\newcommand{\NN}{\mathcal{N}}

\newcommand{\PP}{\mathcal{P}}

\renewcommand{\SS}{\mathcal{S}}

\newcommand{\WW}{\mathcal{W}}

\makeatletter
\def\upintkern@{\mkern-7mu\mathchoice{\mkern-3.5mu}{}{}{}}
\def\upintdots@{\mathchoice{\mkern-4mu\@cdots\mkern-4mu}%
	{{\cdotp}\mkern1.5mu{\cdotp}\mkern1.5mu{\cdotp}}%
	{{\cdotp}\mkern1mu{\cdotp}\mkern1mu{\cdotp}}%
	{{\cdotp}\mkern1mu{\cdotp}\mkern1mu{\cdotp}}}

\newcommand{\UpMultiIntegral}[1]{%
	\edef\ints@c{\noexpand\upintop
		\ifnum#1=\z@\noexpand\upintdots@\else\noexpand\upintkern@\fi
		\ifnum#1>\tw@\noexpand\upintop\noexpand\upintkern@\fi
		\ifnum#1>\thr@@\noexpand\upintop\noexpand\upintkern@\fi
		\noexpand\upintop
		\noexpand\ilimits@
	}%
	\futurelet\@let@token\ints@a
}
\makeatother

\DeclareFontFamily{OMX}{mdbch}{}
\DeclareFontShape{OMX}{mdbch}{m}{n}{ <->s * [0.8]  mdbchr7v }{}
\DeclareFontShape{OMX}{mdbch}{b}{n}{ <->s * [0.8]  mdbchb7v }{}
\DeclareFontShape{OMX}{mdbch}{bx}{n}{<->ssub * mdbch/b/n}{}

\DeclareSymbolFont{uplargesymbols}{OMX}{mdbch}{m}{n}
\SetSymbolFont{uplargesymbols}{bold}{OMX}{mdbch}{b}{n}
\DeclareMathSymbol{\upintop}{\mathop}{uplargesymbols}{82}
\DeclareMathSymbol{\upointop}{\mathop}{uplargesymbols}{"48}

\DeclareFontEncoding{MDB}{}{}
\DeclareFontFamily{MDB}{mdbch}{}
\DeclareFontShape{MDB}{mdbch}{m}{n}{ <->s * [0.8]  mdbchrmb }{}
\DeclareFontShape{MDB}{mdbch}{b}{n}{ <->s * [0.8]  mdbchbmb }{}
\DeclareFontShape{MDB}{mdbch}{bx}{n}{<->ssub * mdbch/b/n}{}
\DeclareFontSubstitution{MDB}{cmr}{m}{n}
\DeclareSymbolFont{mathdesignB}{MDB}{mdbch}{m}{n}%
\SetSymbolFont{mathdesignB}{bold}{MDB}{mdbch}{b}{n}%
\DeclareMathSymbol{\upintclockwise}{\mathop}{mathdesignB}{128}
\DeclareMathSymbol{\upointclockwise}{\mathop}{mathdesignB}{130}
\DeclareMathSymbol{\upointctrclockwise}{\mathop}{mathdesignB}{132}
\DeclareMathSymbol{\upoiint}{\mathop}{mathdesignB}{134}
\DeclareMathSymbol{\upoiiint}{\mathop}{mathdesignB}{136}

\makeatletter
\newcommand{\upint}{\DOTSI\upintop\ilimits@}
\newcommand{\upoint}{\DOTSI\upointop\ilimits@}
\makeatother

\renewcommand{\int}{\upint}

\makeatletter
\renewcommand{\P}{%
	\@ifnextchar\bgroup%
	{\@Pwithargs}
	{\@Pnoargs}
}
\newcommand{\@Pwithargs}[1]{%
	\@ifnextchar\bgroup%
	{\@Ptwoargs{#1}}
	{\@Ponearg{#1}}
}
\newcommand{\@Pnoargs}{\mathbb{P}}
\newcommand{\@Ponearg}[1]{\mathbb{P}\left[ #1 \right]}
\newcommand{\@Ptwoargs}[2]{\mathbb{P}_{#1}\left[ #2 \right]}
\newcommand{\E}{%
	\@ifnextchar\bgroup%
	{\@Ewithargs}
	{\@Enoargs}
}
\newcommand{\@Ewithargs}[1]{%
	\@ifnextchar\bgroup%
	{\@Etwoargs{#1}}
	{\@Eonearg{#1}}
}
\newcommand{\@Enoargs}{\mathbb{E}}
\newcommand{\@Eonearg}[1]{\mathbb{E}\left[ #1 \right]}
\newcommand{\@Etwoargs}[2]{\underset{#1}{\mathbb{E}}\left[ #2 \right]}
\makeatother
\newcommand{\cov}{\mathrm{cov}}

\newcommand{\var}{\mathrm{var}}
\newcommand{\normal}{\mathcal{N}}

\newcommand{\normaldof}[2]{\normal^{#1} \! \left( #2 \right) }

\newcommand{\eqd}{\overset{d}{=}}
\newcommand{\Rad}{\mathop{\mathrm{Rad}}}
\newcommand{\Radhalf}{\Rad\! \left( 1/2 \right)}
\newcommand{\Radhalfd}[1]{\Rad^{#1}\! \left( 1/2 \right)}

\newcommand{\Uni}{\mathrm{Uni}}
\newcommand{\Exp}{\mathrm{Exp}}

\newcommand{\NNet}{\mathrm{NN}}

\newcommand{\lnormm}[2]{\big\Vert #1 \big\Vert_{#2}}
\newcommand{\ltwonormm}[1]{\lnormm{#1}{2}}
\newcommand{\lnorm}[2]{\left\Vert #1 \right\Vert_{#2}}
\newcommand{\ltwonorm}[1]{\lnorm{#1}{2}}
\newcommand{\reals}{\mathbb{R}}
\newcommand{\naturals}{\mathbb{N}}

\newcommand{\vb}{\mathbf{b}}

\newcommand{\vh}{\mathbf{h}}

\newcommand{\vw}{\mathbf{w}}
\newcommand{\vx}{\mathbf{x}}

\newcommand{\vz}{\mathbf{z}}
\newcommand{\vA}{\mathbf{A}}

\newcommand{\vD}{\mathbf{D}}

\newcommand{\vW}{\mathbf{W}}
\newcommand{\vX}{\mathbf{X}}

\newcommand{\vxi}{\boldsymbol{\xi}}

\newcommand{\concat}{\mathord\Vert}

\newcommand{\tb}{\tilde{b}}

\newcommand{\tA}{\tilde{A}}

\newcommand{\tX}{\tilde{X}}

\mathchardef\mhyphen="2D

\newcommand{\bsz}{\mathtt{\scriptsize{bsz}}}

\DeclareMathSymbol{\shortcol}{\mathord}{operators}{"3A}

\newcommand{\Rd}{\reals^d}

\newcommand{\Aijst}{A^{(i,j)}_{s,t}}

\newcommand{\Aij}{A^{(i,j)}}

\newcommand{\Wi}{W^{(i)}}
\newcommand{\Wj}{W^{(j)}}
\newcommand{\bij}{b^{(i,j)}}
\newcommand{\Hi}{H^{(i)}}
\newcommand{\Hj}{H^{(j)}}
\newcommand{\Bi}{B^{(i)}}
\newcommand{\Bj}{B^{(j)}}

\renewcommand{\bar}{\widebar}
\renewcommand{\hat}{\widehat}
\renewcommand{\tilde}{\widetilde}
\makeatletter
\let\save@mathaccent\mathaccent
\newcommand*\if@single[3]{%
  \setbox0\hbox{${\mathaccent"0362{#1}}^H$}%
  \setbox2\hbox{${\mathaccent"0362{\kern0pt#1}}^H$}%
  \ifdim\ht0=\ht2 #3\else #2\fi
  }
\newcommand*\rel@kern[1]{\kern#1\dimexpr\macc@kerna}
\newcommand*\widebar[1]{\@ifnextchar^{{\wide@bar{#1}{0}}}{\wide@bar{#1}{1}}}
\newcommand*\wide@bar[2]{\if@single{#1}{\wide@bar@{#1}{#2}{1}}{\wide@bar@{#1}{#2}{2}}}
\newcommand*\wide@bar@[3]{%
  \begingroup
  \def\mathaccent##1##2{%
    \let\mathaccent\save@mathaccent
    \if#32 \let\macc@nucleus\first@char \fi
    \setbox\z@\hbox{$\macc@style{\macc@nucleus}_{}$}%
    \setbox\tw@\hbox{$\macc@style{\macc@nucleus}{}_{}$}%
    \dimen@\wd\tw@
    \advance\dimen@-\wd\z@
    \divide\dimen@ 3
    \@tempdima\wd\tw@
    \advance\@tempdima-\scriptspace
    \divide\@tempdima 10
    \advance\dimen@-\@tempdima
    \ifdim\dimen@>\z@ \dimen@0pt\fi
    \rel@kern{0.6}\kern-\dimen@
    \if#31
      \overline{\rel@kern{-0.6}\kern\dimen@\macc@nucleus\rel@kern{0.4}\kern\dimen@}%
      \advance\dimen@0.4\dimexpr\macc@kerna
      \let\final@kern#2%
      \ifdim\dimen@<\z@ \let\final@kern1\fi
      \if\final@kern1 \kern-\dimen@\fi
    \else
      \overline{\rel@kern{-0.6}\kern\dimen@#1}%
    \fi
  }%
  \macc@depth\@ne
  \let\math@bgroup\@empty \let\math@egroup\macc@set@skewchar
  \mathsurround\z@ \frozen@everymath{\mathgroup\macc@group\relax}%
  \macc@set@skewchar\relax
  \let\mathaccentV\macc@nested@a
  \if#31
    \macc@nested@a\relax111{#1}%
  \else
    \def\gobble@till@marker##1\endmarker{}%
    \futurelet\first@char\gobble@till@marker#1\endmarker
    \ifcat\noexpand\first@char A\else
      \def\first@char{}%
    \fi
    \macc@nested@a\relax111{\first@char}%
  \fi
  \endgroup
}
\makeatother

\newcommand{\chencomb}{\operatorname{Chen-combine}}

\newcommand{\LAdim}{\tfrac{d(d-1)}{2}}

\newcommand{\tvb}{\tilde{\vb}}
\newcommand{\tvA}{\tilde{\vA}}

\newcommand{\BF}{\operatorname{BF}}
\newcommand{\BFf}[1]{\BF \! \left( #1 \right) }
\newcommand{\Pbf}{\P_{\mathrm{BF}}^{\theta,w}}
\newcommand{\Pbfj}{\P_{\mathrm{BF}}^{\theta}}
\newcommand{\ASP}{\mathrm{ASP}}
\newcommand{\lambASP}{\lambda_{\ASP}}

\newcommand{\echen}{\boldsymbol{\varepsilon}_{\text{\tiny Chen}}}
\newcommand{\elevy}{\boldsymbol{\varepsilon}_{\text{\tiny Lévy}}}

\usepackage{amsmath}
\usepackage{amssymb}
\usepackage{mathtools}
\usepackage{amsthm}
\usepackage{graphicx}
\usepackage{adjustbox}
\usepackage[utf8]{inputenc} 
\usepackage[T1]{fontenc}    
\usepackage{hyperref}       
\usepackage{xcolor}
\definecolor{newBlue}{RGB}{0, 50, 255}
\hypersetup{
    colorlinks,
    citecolor=green,
    linkcolor=newBlue,
    urlcolor=purple
}
\usepackage{url}            
\usepackage{booktabs}       
\usepackage{amsfonts}       
\usepackage{nicefrac}       
\usepackage{microtype}      
\usepackage{xcolor}         
\usepackage{thmtools}
\usepackage{thm-restate}
\usepackage[capitalize,noabbrev]{cleveref}
\usepackage{caption}
\usepackage{subcaption}
\usepackage{multirow}

\hypersetup{
  pdftitle={Generative Modelling of Lévy Area for High Order SDE Simulation}
}

\theoremstyle{plain}

\newtheorem{theorem}{Theorem}[section]

\newtheorem{proposition}[theorem]{Proposition}
\newtheorem{lemma}[theorem]{Lemma}

\theoremstyle{definition}
\newtheorem{definition}[theorem]{Definition}

\theoremstyle{remark}
\newtheorem{remark}[theorem]{Remark}

\usepackage[textsize=tiny]{todonotes}

\newcommand{\m}{\hspace{0.25mm}}
\newcommand*\samethanks[1][\value{footnote}]{\footnotemark[#1]}

\title{Generative Modelling of L\'evy Area\\ for High Order SDE Simulation}

\author{%
  Andra\v{z} Jelin\v{c}i\v{c}\m\thanks{Equal contribution.}\\
  Department of Mathematical Sciences\\
  University of Bath\\
  \texttt{aj2382@bath.ac.uk} 
  \And
  Jiajie Tao\m\samethanks\\
  Department of Mathematics\\
  University College London \\
  \texttt{jiajie.tao.21@ucl.ac.uk} 
  \And
  William F. Turner\m\samethanks\\
  Department of Mathematics\\
  Imperial College London\\
  \texttt{william.turner17@imperial.ac.uk} 
  \And
  Thomas Cass\\
  Department of Mathematics\\
  Imperial College London\\
  \texttt{thomas.cass@imperial.ac.uk} 
  \And
  James Foster \\
  Department of Mathematical Sciences\\
  University of Bath\\
  \texttt{jmf68@bath.ac.uk} 
  \And
  Hao Ni \\
  Department of Mathematics\\
  University College London \\
  \texttt{h.ni@ucl.ac.uk} 
}

\usepackage[style=alphabetic, maxbibnames = 99, giveninits=true]{biblatex}
\usepackage{csquotes}
\begin{document}

\maketitle
\begin{abstract}
It is well understood that, when numerically simulating SDEs with general noise, achieving a strong convergence rate better than $O(\sqrt{h})$ (where $h$ is the step-size) requires the use of certain iterated integrals of Brownian motion, commonly referred to as its ``L\'{e}vy areas''. However, these stochastic integrals are difficult to simulate due to their non-Gaussian nature and for a $d$-dimensional Brownian motion with $d > 2$, no fast almost-exact sampling algorithm is known.

In this paper, we propose L\'{e}vyGAN, a deep-learning-based model for generating approximate samples of L\'{e}vy area conditional on a Brownian increment. Due to our ``Bridge-flipping'' operation, the output samples match all joint and conditional odd moments exactly. Our generator employs a tailored GNN-inspired architecture, which enforces the correct dependency structure between the output distribution and the conditioning variable. Furthermore, we incorporate a mathematically principled characteristic-function based discriminator. Lastly, we introduce a novel training mechanism termed ``Chen-training'', which circumvents the need for expensive-to-generate training data-sets. This new training procedure is underpinned by our two main theoretical results.

For $4$-dimensional Brownian motion, we show that L\'{e}vyGAN exhibits state-of-the-art performance across several metrics which measure both the joint and marginal distributions. We conclude with a numerical experiment on the log-Heston model, a popular SDE in mathematical finance, demonstrating that high-quality synthetic L\'{e}vy area can lead to high order weak convergence and variance reduction when using multilevel Monte Carlo (MLMC).
\smallbreak
\end{abstract}
\section{Introduction}
The numerical simulation of Stochastic Differential Equations (SDEs) is a ubiquitous task encountered in a wide variety of fields, ranging from mathematical finance \cite{shreve2004finance} and systems biology \cite{browning2020biology} to molecular dynamics \cite{leimkuhler2015molecular} and data science \cite{li2019LangevinMC}. Real-world phenomena arising in these areas are often described well by SDEs formulated through \textit{It\^o calculus} and of the general form:
\begin{equation}\label{eq:SDE_1}
    dX_t = f(X_t)\m dt + \sum_{i=1}^d g_i(X_t)\m dW_t^{(i)},\ \ X_0=x_0,
\end{equation}
where the solution $X = \{X_t\}_{t\in[0,T]}$ takes values in $\R^e$, $W = (W^{(1)},\dots,W^{(d)})$ denotes a standard $d$-dimensional Brownian motion and $f, g_i : \R^e \rightarrow \R^e$ are suitably regular vector fields on $\R^e$. In practice, one is often concerned with estimating quantities of the form $\bar{\varphi}\coloneqq\Ex\left[\varphi(X)\mid X_0=x_0\right]$, where the function $\varphi$ may depend on the whole trajectory of $X$ or simply on the terminal value $X_T$; such as the payoff of a European call-option. On occasion, it may be possible to obtain $\bar{\varphi}$ by solving certain PDEs (e.g.~through the backward Kolmogorov equation or the Feynman-Kac formula).\smallbreak

However, the standard approach is to use Monte Carlo simulation, where one uses a discretisation scheme to generate approximate sample paths $\{\hat{X}_{i}\}_{i=1}^{N}$ of $X$, which can be used to approximate $\bar{\varphi}$ by taking the average of $\{\varphi(\hat{X}_{i})\}_{i=1}^{N}$. Given the importance of Monte Carlo simulation in applications, there is a rich literature concerning numerical methods for SDEs and their theoretical properties. A broad range of discretisation schemes are available, such as the classical Euler-Maruyama and Milstein schemes as well as the higher order Talay \cite{Talay} and Ninomiya-Victoir \cite{NV} schemes. For more details on the numerical simulation of SDEs, we refer the reader to \cite{KP} and \cite{PT}.\smallbreak

There are two standard measures for the effectiveness of numerical schemes: strong error (or MSE) and weak error. It is a well-known result of Clark and Cameron \cite{ClarkCameron} that numerical schemes using only increments of Brownian motion are limited in general to a strong convergence rate of at most $O(\sqrt{h}\,)$, where $h$ denotes the step size. Furthermore, to the best of our knowledge, all numerical schemes achieving second order weak convergence require the generation of random variables in addition to the Brownian increments. In particular, the Talay scheme \cite{Talay} and the Ninomiya-Victoir scheme \cite{NV} require the generation of Rademacher random variables. Further examples include the Ninomiya-Ninomiya scheme \cite{NN} and a stochastic Runge-Kutta method due to R\"o{\ss}ler \cite{RungeKutta}. In all of these cases, the additional random variable(s) are generated to replace certain second order iterated integrals of Brownian motion, which are commonly referred to as its \textit{L\'evy area}.\medbreak

\begin{definition}
The L\'evy area of a $d$-dimensional Brownian motion over $[s,t]$ is a $d\times d$ antisymmetric matrix whose $(i,j)$-th entry is entries given by,
\begin{equation*}
    A_{s,t}^{(i,j)} \coloneqq \frac{1}{2}\bigg(\int_{s}^{t}(W_r^{(i)}-W_{s}^{(i)})\, dW_r^{(j)}-\int_{s}^{t}(W_r^{(j)}-W_{s}^{(j)})\, dW_r^{(i)}\bigg).
\end{equation*}
\vspace*{-4mm}
\end{definition}
\begin{figure}[H]
    \centering
    \includegraphics[width=0.95\textwidth]{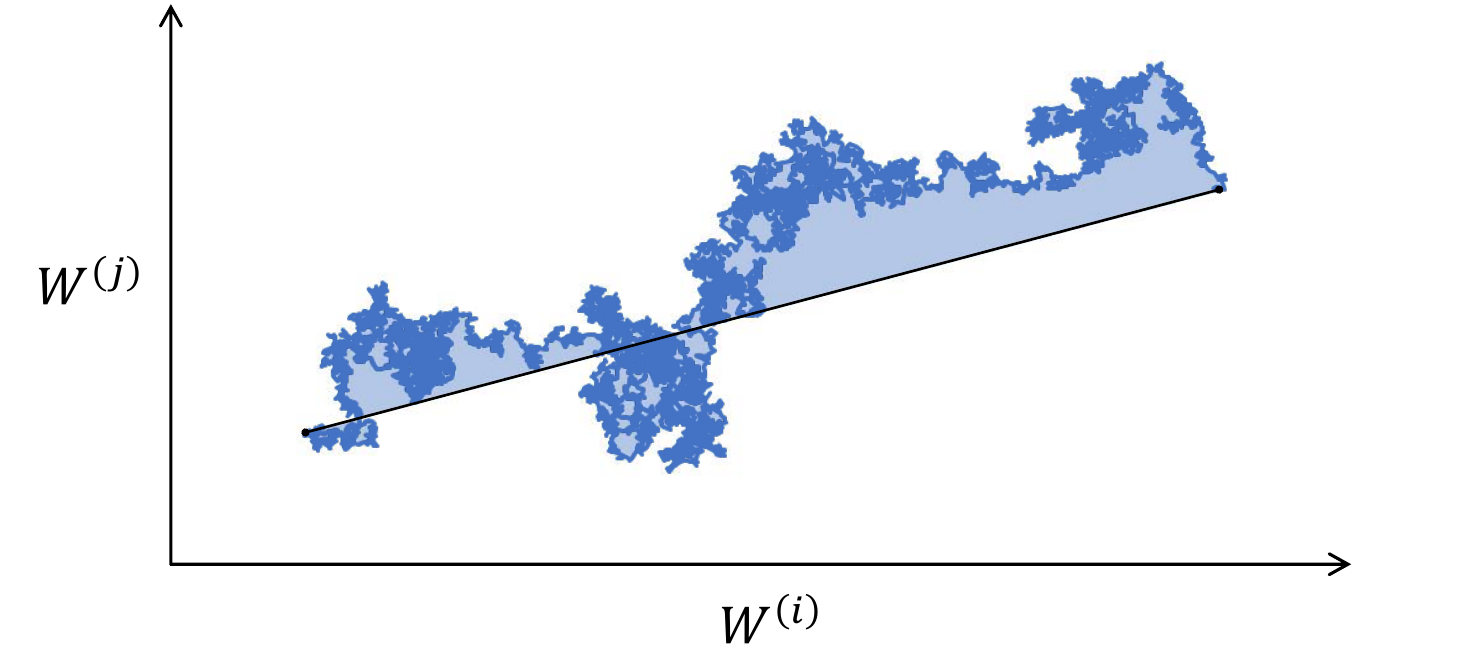}
    \caption{Each entry $A^{(i,j)}$ is the area between the independent Brownian motions $W^{(i)}$ and $W^{(j)}$ (diagram adapted from \cite{FosterThesis}).}
    \label{fig:levy_area}
\end{figure}

When the vector fields of the SDE \cref{eq:SDE_1} do not satisfy the commutativity condition $[g_i,g_j]=0$ (where $[g_i, g_j](x) := g_j^\prime(x)g_i(x) - g_i^\prime(x)g_j(x)$ denotes the standard Lie bracket of vector fields), then schemes that achieve high order Strong convergence, such as the Milstein and log-ODE methods, require the simulation of L\'evy area. Consequently, the approximation of L\'evy area has received much interest in recent decades, with the view towards both high order weak and strong convergence.\smallbreak

Approximations to L\'evy area have been well studied \cite{Davie,Dickinson,GL,KP,KPW,FosterThesis,FH,MR,MW,Wiktorsson}, with the majority of approximations concerning strong estimation. Strong estimators aim to approximate L\'evy area by minimising the mean-squared error so that the resulting estimator may be incorporated into the strong analysis of the discretisation scheme. Typically such estimators rely on truncated expansions of Brownian motion, such as the Fourier series expansion \cite{KP,KPW, KR}, the Karhunen-Lo\`eve expansion \cite{Loeve} and more recently the polynomial expansion \cite{foster2025levy, FH, foster2020b}. These estimators are often improved by estimating the tail sum of the of the expansion in an appropriate manner, see for example \cite{MR,Wiktorsson}.\smallbreak

To the best of our knowledge, there is no known scheme which simulates L\'evy area exactly, with even the ``rectangle-wedge-tail'' algorithm of Gaines and Lyons \cite{GL} requiring numerical integration. Moreover, this approach is only applicable in $d=2$. On the other hand, the main drawback of the truncated expansion methods is the cost of simulation. In practice, one is often required to generate millions of L\'evy area samples, and the aforementioned methods often require a high truncation level to achieve good performance. Consequently, in recent years, there has been a renewed focus on approximations of L\'evy area that are suitable for weak discretisation schemes, where the estimator is less costly to generate. The aim of a weak estimator is to match some moments of the L\'evy area given a Brownian increment. These estimators differ in complexity depending on their intended usage. Basic estimators include the Rademacher random variables that appear in the Talay scheme \cite{Talay} and Davie's approximation \cite{Davie,FL} which uses a Gaussian random variable with the correct variance (this may be improved to have the correct conditional variance given the Brownian increment \cite{FosterThesis}). Perhaps the most sophisticated weak approximation is due to Foster \cite{FosterThesis}, which matches the first five conditional moments of L\'evy area given the Brownian increment when $d\leq3$.\smallbreak

In this article, we provide a new approach to the construction of weak estimators of L\'evy area given a Brownian increment through the use generative modelling techniques. To the best of our knowledge, this is the first time that the powerful toolkits provided by modern machine learning have been applied to the problem of L\'evy area simulation. Arguably the main obstacle to this approach is the computation effort required to generate considerable amounts of precise samples of L\'evy area. We note this is in principle possible, through the use of a truncated Fourier series method \cite{KR}, with other options also possible. However, in the context of L\'evy area generation, we present a novel training algorithm based on Chen's relation \cite{Chen57} that allows for the training of a generative model without access to a dataset of L\'evy area samples.
\subsection{Our contributions}
In this article we present L\'evyGAN, a deep-learning based generative model that simulates the L\'evy area of arbitrary dimensional Brownian motions. Deep-learning based generative models have been widely used for data synthesis where a parametric model is trained to learn the target distribution. Among the variety of generative models, GAN-typed models \cite{Mirza2014}, \cite{Goodfellow2014} have been particularly noteworthy for their performances. GAN, short for Generative Adversarial Network, operates on the compelling principle of adversarial training. This approach consists of two neural networks - a generator and a discriminator - that are trained simultaneously. The generator's task is to create synthetic data, while the discriminator's role is to distinguish between real and generated data.\smallbreak

With no exception, GAN-type models also possess drawbacks like other generative models. The necessity of real data as reference sets for training is one of them. Machine learning models are often data-driven and sometimes data-greedy; normally, practitioners collect real-world data and approximate its distribution using an empirical distribution outputed by the generator. This is especially pertinent in the context of L\'evy area generation, where a competitive method needs to achieve very high accuracy, which incurs a high statistical complexity, and thus requires large amounts of data. In contrast to standard GANs, score-based diffusion models and variational autoencoders, L\'evyGAN is designed to learn the target distribution without requiring any samples from it at all. We term this approach ``Chen-training'', which is theoretically underpinned by the unique invariance of the joint law of Brownian motion and L\'evy area under Chen's relation, see \cref{thm:DistrChenUniquenessGeneral}.\smallbreak

We have designed both the generator and discriminator by exploring the features of the joint law of Brownian motion and its L\'evy Area \cite{FH, foster2020b, FosterThesis}. In particular, for the generator, we ensure that the joint distribution of our generator is permutation invariant and that each component of the generated area depends only on the relevant components of the Brownian increment. We also ensure that all odd cross moments of our generator are exact through the multiplication by certain Rademacher random variables. For the discriminator, we have chosen a characteristic function based discriminator, initially proposed in \cite{Ansari2019, CFGAN}. Inspired by this approach and the more general method of  \cite{HangPCF}, we define the unitary characteristic function of a random variable as a generalization of the characteristic function onto higher degree Lie algebra.\smallbreak

For our numerical results, we train the L\'evyGAN in $d=4$. It is noteworthy that the model is able to generate the L\'evy area for arbitrary Brownian dimensions, with no loss in performance guaranteed in $d^\prime<d$ and empirically strong performance for $d^\prime >d$. Empirically, we show that L\'evyGAN attained the best performance among other weak estimators, such as those found in \cite{Davie} and \cite{FosterThesis}, in terms of distributional metrics and is comparable in generation speed to the method in \cite{FosterThesis}. Finally we provide an application of weak approximations to L\'evy area to high order multi-level Monte Carlo numerical schemes. In this example, we demonstrate that the inclusion of an approximate L\'evy area term in the Strang splitting method achieves higher order variance reduction and weak convergence. Moreover, we provide evidence that an estimator that only matches the variance of L\'evy area (such as the Rademacher random variables found in the Talay scheme) is not appropriate for this application. The L\'evyGAN implementation, together with a trained model for $d=4$ can be found at \href{https://github.com/andyElking/LevyGAN}{github.com/andyElking/LevyGAN}.

\subsection{Outline and Common Notation}
This article is divided into five main sections. In \cref{sec:GAN} we recall the standard setup of generative adversarial networks and discuss why the traditional approach to generative modelling is not easily applicable to our setting. In \cref{sec:generator} we outline the structure of our generator. This section focuses on the symmetries of the joint law of Brownian motion and L\'evy area that we hard code into our generator. This includes so-called ``bridge-flipping'', a precise multiplication by certain Rademacher random variables to ensure all joint odd moments are correctly estimated and to help the generator train evenly across all quadrants in space. In \cref{sec:PairNet} we introduce a network architecture dubbed ``pair-net" inspired by graph neural networks that ensures the correct dependence structure between the coordinates of Brownian motion and the coordinates of L\'evy area. The structure of our discriminator is outlined in \cref{sec:Discriminator}. Here we discuss two alternatives for a loss function based on the analytical form of the joint characteristic function of L\'evy area and Brownian motion and a generalisation termed the unitary characteristic function proposed in \cite{HangPCF}. Our novel training approach ``Chen-training'' is covered in \cref{sec:ChenTraining}, before the whole training procedure is summarised in \cref{sec:levygan}. Finally, in \cref{sec:Results}, we compare the distributional performance of our generator in comparison to the state-of-the-art Foster method \cite{FosterThesis} and demonstrate the applicability of weak estimators for L\'evy area to multilevel Monte-Carlo.\smallbreak

To conclude the introduction, we outline some common notation to be used throughout the article.\smallbreak
    \begin{enumerate}[label=\arabic*)]
        \item $\P_{X\mid Y=y}$ for the conditional distribution of $X$ given $Y=y$.
        \item With abuse of notation $(X\mid Y=y)\eqd (Z\mid Y=y)$ if $\P_{X\mid Y=y}=\P_{Z\mid Y=y}$.
        \item $(W_t)_{t\in [0,1]}$ for a $d$-dimensional Brownian motion, and the process
        $(A_t)_{t\in [0,1]}$ with $A_t=\big\{A^{(i,j)}_t\big\}_{1\leq i<j\leq d}$ a $\frac{d(d-1)}{2}$-dimensional vector representing the flattened upper triangle of the L\'evy area matrix associated with the Brownian motion. Unless stated otherwise, we denote by $a$ the dimension of L\'evy area, i.e. $a = \frac{d(d-1)}{2}$.
        \item For any process $(X_t)_{t\in [0,1]}$, $X_t$ denotes the process evaluated at time $t$.
        \item $\P_{(W_t,A_t)}$ for the joint law of Brownian motion and L\'evy area at time $t$. 
        \item $\NN^d(\mu, \sigma^2)$ for a $d$-dimensional Gaussian distribution with independent coordinates, each with mean $\mu$ and variance $\sigma^2$.
        \item $\mathrm{Rad}^d(p)$ for a $d$-dimensional Rademacher random variable with independent entries each of which takes $1$ and $-1$ with probability $p$ and $1-p$ respectively.
        \item $\Phi_X$ for the characteristic function of a random variable $X$.
        \item $W_{s,t}\coloneqq W_t-W_s$ for the increments of a Brownian motion $(W_t)_{t\in [0,1]}$, and we use $W_t=W_{0,t}$ interchangeably.
        \item For any \textit{tensor}, i.e.~an element $\vx \in \R^{\cdot \times d}$, we always denote by $x^{(i)}\in \R^{\cdot}$ the $i$\textsuperscript{th} coordinate of the second dimension.
    \end{enumerate}
\section{The GAN Architecture}
\label{sec:GAN}
The goal of this article is to build an efficient and accurate estimator that approximates the conditional law $\P_{A_t|W_{t}}$, and hence the joint law, $\P_{(W_{t}, A_t)}$, of a Brownian increment and its L\'evy area. Thanks to the scaling property of Brownian motion, it is enough to consider the problem when $t = 1$. \smallbreak
 
 We adopt a conditional GAN (Generative Adversarial Network) approach as proposed by \cite{Mirza2014}. GAN-typed models, initially proposed by \cite{Goodfellow2014} consist of a pair of competing neural networks - the generator and the discriminator. The aim of the generator is to create ``fake'' data $\tilde{\vx}$, from some noise distribution, trying to mimic a target distribution, while the discriminator will be given both $\tilde{\vx}$ and data $\vx$ from the true distribution and will try to distinguish the ground truth between them. The dynamics between the generator and discriminator are controlled by a min-max game acting on a loss function, which usually represents the distance between two distributions. In a conditional GAN, the generator is given not only samples from the noise distribution, but also the conditioning variable -- in our case, this is the Brownian increment $W$, or later the space-time L\'{e}vy area $H$ (see \cref{sec:BridgeFlipping}). Here we provide the description of a classical conditional GAN adapted to our interest.

\begin{definition}[Classical conditional GAN for L\'{e}vy Area generation]
    \label{def:classicalGAN}
    Let $d \geq 2$ be the Brownian dimension, and $a \coloneqq \LAdim$ be the dimension of the associated L\'{e}vy area vector. Assume $z$ is an $n$-dimensional noise vector distributed according to $\P_{z}$. The conditional generator $G_\theta$ and the discriminator $D_{\eta}$ are maps
    \[
        G_\theta: \R^{d}\times\R^n \rightarrow \R^d\times\R^{a}\, \;\; D_\eta: \R^{d}\times\R^a \rightarrow \R,
    \]
   that are parametrized by $\theta$ and $\eta$ respectively. When restricted to the first $d$-coordinates of the domain and co-domain, $G_\theta$ is enforced to be identity, i.e. $G_\theta(w,z)=(w,\tilde{A})$. Let $\P_{(W_1, A_1)}$ be the ground truth distribution of the coupled process. An example loss function might be given by
    \begin{equation*}
        L(\theta, \eta) \coloneqq \E_{(w, A)\sim \P_{(W_{1}, A_1)}} \left[D_{\eta}(w, A)\right]-\E_{\;w \sim\P_{W_1},z\sim \P_{z}}\left[D_{\eta}(G_{\theta}(w ,z))\right],
    \end{equation*}
    where one restricts $D_\eta$ to be at most $1$-Lipschitz. The models are trained using the min-max game $\min_{\theta} \max_{\eta} L(\theta, \eta)$ until convergence. The generator obtained is then used to simulate L\'{e}vy areas.
\end{definition}

Since there are no known methods for exact simulation of L\'evy areas in $d>2$, the ``true'' L\'evy area samples $\vA \sim \P_{A_1|W_{1}}$ must themselves be obtained through approximate simulation. In particular, we might generate the ``true'' L\'{e}vy area samples using a Julia-language package created by \cite{KR}, which complements their paper about approximate strong simulation of L\'evy area through truncated Fourier series methods. We also present the flowchart of this methodology in \cref{fig:Lévy_gan_scheme}.\vspace*{-3mm}
\begin{figure}[H]
    \centering
    \includegraphics[width=0.95\textwidth]{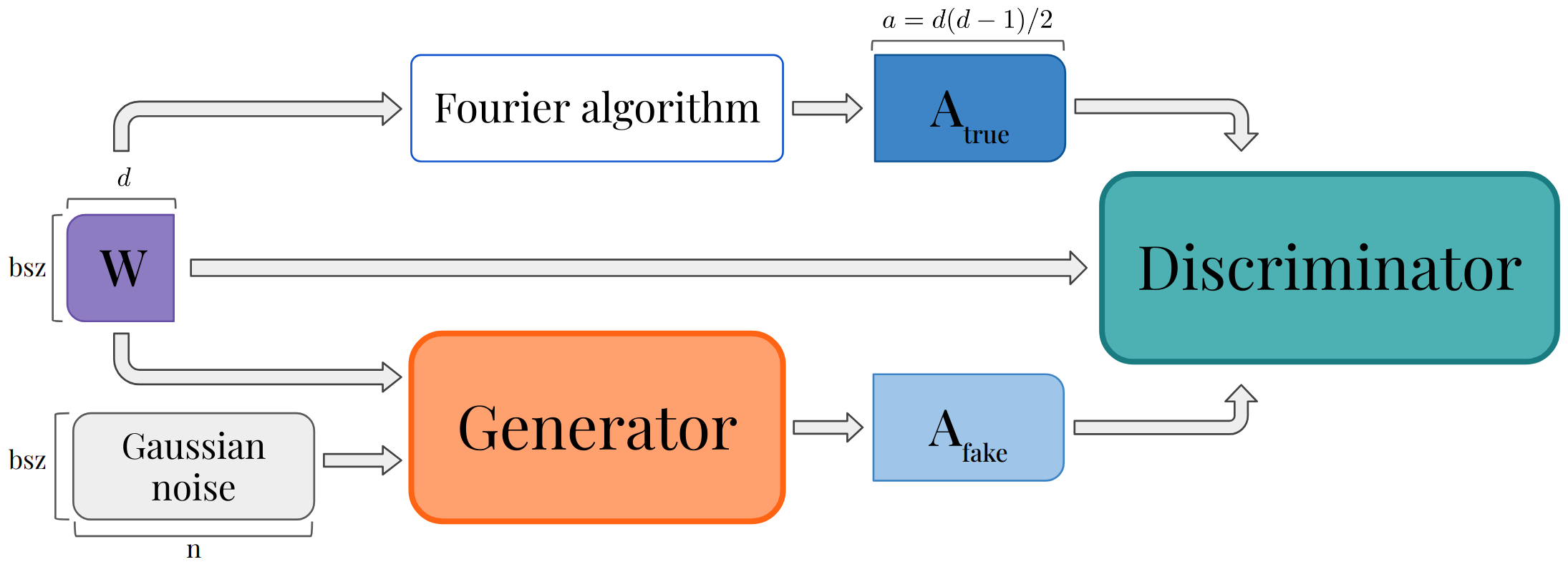}
    \caption{A schematic of L\'{e}vy generation for classical conditional GAN. Throughout this article, \texttt{bsz} represents the training batch-size.}
    \label{fig:Lévy_gan_scheme}
\end{figure}
The main drawback of this methodology is the need for simulating approximate real samples of L\'{e}vy area. Using real data not only slows the training procedure but also introduces simulation error (due to the truncation in the Fourier series) and finite sample error (leading to overfitting). In order to address this problem, we propose a novel approach, L\'{e}vyGAN, which completely excludes real samples from the training process and is well justified by our two main theoretical contributions \cref{thm:DistrChenUniquenessGeneral} and \cref{thm:chen_err_bd}. 


\section{Generator}
\label{sec:generator}
In order to improve the accuracy of the generated distribution 
we can consider the symmetries of the L\'{e}vy area distribution and hard-code them into the generator itself. This way the generator will consist of both a neural net, and additional operations applied to the network's output. One symmetry of L\'evy area that is desirable to reflect is the fact that its distribution is mean zero when conditioned on any increment $W_{0,1}=w$. We do this through two operations, which combined we term ``bridge-flipping''.

\subsection{Bridge-Flipping}\label{sec:BridgeFlipping}

We would like to hard-code the symmetry of L\'evy area about zero into the generator's architecture. Even though each dimension of L\'evy area $\Aij$ is symmetric about zero, their joint distribution is not invariant to multiplying any individual dimension by $-1$. The dimensions of the underlying Brownian motion, however, are independent, and hence each can be mirrored separately without violating their joint law.\smallbreak

The goal is hence to find a symmetry of L\'evy area corresponding to independently flipping individual dimensions of the Brownian motion. Notice, however, that we are trying to generate L\'evy area conditional on a fixed input $W_{0,1} = w$, and so we do not wish to flip the increment of Brownian motion itself. We can circumvent this issue by considering the polynomial expansion of L\'evy area \cite{FH}, which decomposes the Brownian motion into the components dependent on $w$ and components independent of $w$, the latter of which can then be flipped independently. To this end, we first define the Brownian bridge and its accompanying ``space-time'' and ``space-space'' L\'{e}vy areas.
\begin{definition}[Brownian bridge]
\label{def:BBridge}
Let $0 \leq s \leq t < \infty$. Then the Brownian bridge of $W$ on $[s,t]$ returning to zero at time $t$ is defined as
\[
B_{s,u} \coloneqq W_{s,u} - \frac{u-s}{t-s} W_{s,t} \;\;\;\; \text{for } \; u \in [s,t].
\]
The ``space-time'' L\'{e}vy area $H_{s,t} \in \reals^d$ and ``space-space'' L\'{e}vy area $b_{s,t} \in \reals^{d \times d}$ of $B$ over $[s,t]$ are
\begin{align}
    H^{(i)}_{s,t} &\coloneqq \frac{1}{t-s} \int_s^t \Bi_{s,u} \, du = \frac{1}{t-s} \int_s^t \Wi_{s,u} - \frac{u-s}{t-s} \Wi_{s,t} \, du \;\;\; \text{for } 1 \leq i \leq d, \\[1mm]
    b^{(i,j)}_{s,t} &\coloneqq \int_s^t B_{s,u}^{(i)} \,\circ dB_u^{(j)} \;\;\; \text{for } 1 \leq i,j \leq d.
\end{align}

Whenever $s=t$, we define $b_{s,t}=H_{s,t}=0$. We write $H_t$ and $b_t$ for $H_{0,t}$ and $b_{0,t}$ respectively.
\end{definition}

It turns out that $(H, b)$ and $W_{s,t}$ are independent, the marginal distribution of $H$ is Gaussian, and the marginal of $b$ is logistic.

\begin{proposition}[Distribution of Brownian bridge L\'{e}vy area \cite{foster2020b,FosterThesis}]\label{prop:BBLAdistn}
    For fixed $0 \leq s < t < \infty,$ the process $\{(H_{s,u}, b_{s,u})\}_{u \in [s,t]}$ and the increment $W_{s,t}$ are independent. Furthermore, $H$ is distributed as a $d$-dimensional Gaussian with independent coordinates and the marginal distribution of each Brownian bridge L\'evy area is logistic:
    \[
    H_{s,t} \sim \normal^d \left( 0, \frac{1}{12} (t-s) \right)\quad\text{and}\quad b_{s,t}^{(i,j)}\sim\mathrm{Logistic}\left(0,\frac{1}{2\pi}(t-s)\right).
    \]
\end{proposition}
This yields the first two terms of the polynomial expansion of L\'{e}vy area.\medbreak
\begin{proposition}[Polynomial expansion of L\'{e}vy area \cite{FH}]\label{prop:PolyExpansion}
    The L\'evy area of a $d$-dimensional Brownian motion $W$ has the following decomposition:
    \begin{equation*}
        A_{s,t} = H_{s,t} \otimes W_{s,t} - W_{s,t} \otimes H_{s,t} + b_{s,t},
    \end{equation*}
    where $\otimes$ denotes the outer-product of vectors.
\end{proposition}

This decomposition reduces the conditional generative task to the estimation of the Brownian bridge L\'evy area conditional on the Brownian increment, where the target distribution and conditioning variable are independent. This approach may be generalised to the estimation of the tail sum of the polynomial expansion of L\'evy area truncated at a higher level. Since $B$ and the increment $W_{0,1} = w$ are independent, the conditional distribution $ \left( B_t \mid W_{0,1} = w \right) $ is symmetric around 0. Furthermore, each dimension of $B$ is its own independent process, so we can flip each individually without affecting the distribution, as established in the following lemma. 
\begin{lemma}\label{lem:BridgeLAFlip}
    Let $W$ be a $d$-dimensional Brownian motion on $[0,1]$ and let $B, H, b$ be the corresponding derived processes from \cref{def:BBridge}. Fix some $\xi \in \{-1, 1 \}^d $, and let $H^\prime$, and $b^\prime$ be the space-time and space-space L\'evy area processes associated with the process $\xi \odot B = \{\xi \odot B_t\}_{t\in[0,t]}$, where $\odot$ denotes the Hadamard (coordinate-wise) product. Then
    \[
         \{ B_t \}_{t \in [0,1]} \eqd \{ \xi \odot B_t \}_{t \in [0,1]}, \;\;  H^\prime = \xi \odot H, \;\; b^\prime = (\xi \otimes \xi) \odot b, \;\; \text{ and } \;\; (H^\prime,b^\prime)\eqd (H, b).
    \]
\end{lemma}
We also include a multiplication of a final independent random variable $\xi_0\sim\mathrm{Rad}(\tfrac{1}{2})$, whose role is explained in \cref{prop:moments}. Combining this with \cref{lem:BridgeLAFlip} we obtain the ``bridge-flipping'' function.
\begin{definition}
Let $w, h, \xi \in \R^d$, $b\in \R^{d\times d}$ and $\xi_0\in \R$. Then the Bridge-flipping function is defined as
\begin{equation}
\label{eqn:BF}
\BFf{ w, h, b, \xi_0, \xi } \coloneqq \xi_0 \left( \left( \xi \odot h \right) \otimes w - w \otimes \left( \xi \odot h \right) + (\xi \otimes \xi) \odot b \right),
\end{equation}
\end{definition}
for which we have the following as a consequence of \cref{prop:PolyExpansion}, \cref{lem:BridgeLAFlip}.
\begin{theorem}[Bridge-flipping]
    \label{thm:Bridge_flipping}
    Let $\,\xi_0, \ldots, \xi_d\stackrel{\textrm{i.i.d.}}{\sim}\mathrm{Rad}(\tfrac{1}{2}) $ be random variables so that $W$, $(H, b)$, and $(\xi_0,\dots,\xi_d)$ are independent. Write $\xi = (\xi_1, \ldots, \xi_d)$ and fix some $w \in \Rd$. Let $H, b$ be as in \cref{def:BBridge}. Then for every $t\in [0,1]$
    \[
     \big(A_{0,t}\mid W_{0,1} = w\big)\eqd \big({\BFf{ W_{0,t}, H_{0,t}, b_{0,t}, \xi_0, \xi }\mid W_{0,1} = w }\big).
    \]
    For our purpose, we will utilise the result for $t=1$:
    \[
         \big({A_{0,1} \mid W_{0,1} = w}\big) \; \eqd \; \BFf{ w, H_{0,1}, b_{0,1}, \xi_0, \xi }.
    \]
\end{theorem}

Recall that $H_{s,t} \sim \normaldof{d}{ 0,  \tfrac{1}{12} (t-s) }$, and that $b$ and $H$ are correlated, but that $(H, b)$ and $W_{0,1}$ are independent. Hence, given a neural net $\NNet_{\theta} : \R^{d+n} \rightarrow \Rd$, we define the ``Bridge-flipping generation":

\begin{algorithm}[H]
\caption{Area generation using Bridge-flipping}\label{alg:BFGen}

\textbf{Input:} $\theta$ - neural net parameters, $d$ - Brownian dimension, $w$ - Brownian increment, $n$ - dimension of noise vector
\begin{algorithmic}[1]
\State $\xi_0 \gets \Radhalf; \;\; \xi \gets \Radhalfd{d}$
\State $h \gets \normaldof{d}{ 0,  \tfrac{1}{12} (t-s) } ; \;\; z \gets \normaldof{n}{0,1}$
\State $\tb \gets \NNet_{\theta} (h , z)$  \Comment{we want $\NNet_{\theta} (h , z) \, \overset{d}{\approx} \, \left( b_{0,1} \mid H_{1}=h \right)$}
\State \Return $\operatorname{BF}\hspace{-0.6mm}\big(w, h, \tb, \xi_0, \xi\big)$
\end{algorithmic}
\end{algorithm}

This construction has some desirable properties.
\begin{enumerate}[label=\arabic*)]
    \item Informally speaking, the use of $\xi$ effectively makes the generator behave identically on all orthants of $\Rd$, and hence any learning done in one orthant transfers equally to the other orthants. This speeds up training and significantly improves the generator's accuracy, as it now perfectly mimics the symmetric structure of L\'evy area.
    \item The neural net can be trained directly on the distribution of $\left( b_{0,1} \mid H_{0,1}=h \right)$, and is then used for generation of $(A_{0,1} | W_{0,1} = w)$ using the BF algorithm.
    \item The structure of BF allows for efficient implementation of back-propagation.\smallbreak
\end{enumerate}

The use of the extra Rademacher random variable $\xi_0$ is to guarantee that all odd joint and conditional moments are correctly matched. This is summarised in \cref{prop:moments}.
Having $\tA$ unbiased means that some of the usual error analysis from stochastic numerics can be applied, such as in the proof of \cref{thm:DistrChenUniquenessGeneral}, where one of the requirements is that $\tA$ be unbiased. Recall that the Milstein scheme requires a subroutine which generates samples of L\'{e}vy area given a Brownian increment. Since $\tA \sim \Pbf$ is unbiased, one can establish theoretical guarantees on the convergence of Milstein's method with the BF generator as this subroutine. The following result can be proven by applying \cite[Theorem 1.1]{Milstein2004} with $p_1 = \tfrac{3}{2}, \; p_2 = 1$.

\begin{proposition}
    \label{prop:MilsteinConvBF}
    Given a time horizon $T > 0$, and a step size $h = \frac{T}{N}$ where $N\geq 1$ is the number of steps, let $\{\tX_n\}_{n \in \{0,\cdots, N\}}$ be the output of Milstein's scheme (see \cref{sec:milstein} for the definition) applied to the SDE
    \[dX_t = \mu(X_t, t)\m dt + \sigma(X_t, t)\m dW_t.\]
    where the L\'evy areas provided as input to Milstein's scheme were generated by the BF generator. Then if $\mu, \sigma$ are continuous and globally Lipschitz, there exists a constant $C > 0$ such that for sufficiently small $h$
    \[
    \sup_{0\m\leq\m n\m\leq N} {\E\Big[\big| \tX_n - X_{nh} \big|^2}\Big]^{\frac{1}{2}} \leq C h^{\frac{1}{2}}.
    \]
\end{proposition}
\begin{remark}
    Although \cref{thm:Bridge_flipping} and \cref{alg:BFGen} leads to the generation of $(A_{0,1} | W_{0,1} = w)$, we emphasize that it can be generalized to the generation of $(A_{0,t} | W_{0,1} = w)$ for any $t\in [0,1]$ by using the scaling property of Brownian motion and its L\'evy area. In particular, given $w$ we can:\smallbreak
    \begin{enumerate}[label = \arabic*)]
        \item sample $w'$ from the distribution of $\big(W_t | W_1 = w)$, i.e.~$w' \sim \mathcal{N} \big( t \, w, \: t (1-t) \big)$.
        \item rescale $w'' = \frac{w'}{\sqrt{t}}$ and sample $a'' \sim \big(A_{0,1}\mid W_1 = w''\big)$.
        \item finally rescale again $a' = t \, a''$ which has the desired distribution $a' \sim \big(A_{0,t}\mid W_1 = w\big)$.
    \end{enumerate}
\end{remark}

\subsection{The Pair-net Generator}
\label{sec:PairNet}


It is clear that for any $1 \leq i,j \leq d$, both $\bij_{1}$ and $\Aij_{1}$ depend only on the paths of $\{\Wi_t\}_{t\in[0,1]}$ and $\{\Wj_t\}_{t\in[0,1]}$, but not on $\{W^{(k)}_t\}_{t\in[0,1]}$ for $k\notin\{i,j\}$. This dependency structure can be well described using graphs, encouraging us to employ model architectures reminiscent of Graph Neural Networks (GNN) \cite{GNN}.\smallbreak

Consider a clique on $d$ nodes, where each node corresponds to one dimension of the Brownian motion, and each edge $(i,j)$ is associated to $\Aij$ (or $\bij$). Unlike GNNs, where it might be desirable for information to propagate throughout the entire graph, we want edge $(i,j)$ to never see information at node $k \not\in \{i,j\}$. Hence, our architecture should function like a 1-step GNN with edge-wise outputs.



So that $\Aij$ only depends on $\Wi$ and $\Wj$, we generate a separate noise vector for each dimension of Brownian motion. We can interpret this as some embedding of the entire path $\{ \Wi_t \}_{t \in [0,1]} \mapsto \mathrm{noise}^{(i)} $, but in practice we use Gaussian noise.

\begin{definition}[Pair-net]
\label{def:pairnet}
    The Pair-net is defined as the mapping:
    \begin{align*}
        \text{PairNN}_{\theta} : (\mathbb{R},\mathcal{Z}) \times (\mathbb{R},\mathcal{Z}) \to \mathbb{R}; \;\;\;\;
        (H, Z) \times (H', Z') \mapsto \tvb,
    \end{align*}
    where $\mathcal{Z}$ be the space of latent noise. Consider $H_{0,1}$ and $b_{0,1}$ from \cref{def:BBridge}. We approximate $\mathbb{P}_{b_{0,1}\mid H_{0,1} = h}$ via 
     \begin{equation*}
         \tvb^{(i,j)} = \text{PairNN}_{\theta}((h^{(i)}, \vz^{(i)}), (h^{(j)}, \vz^{(j)})),\quad 1\leq i<j\leq d
     \end{equation*}
    for any $\vz \in \mathcal{Z}^d$. In practice, we choose $\mathcal{Z}$ to be $\R^n$ and we let $\vz^{(i)}$ be a $n$-dimensional Gaussian noise for $1\leq i\leq d$. Combining this with bridge-flipping, we describe the generation algorithm and flowchart in \cref{alg:PairNet} and \cref{fig:pair_net_scheme}.
\begin{algorithm}[H]
\caption{Area generation using Pair-net generator and Bridge-flipping}\label{alg:PairNet}

\textbf{Input:} $\theta$ - neural net parameters, $d$ - dimension of Brownian motion, $w$ - Brownian increment, $n$ - dimension of noise vector associated with each coordinate of Brownian motion.
\begin{algorithmic}[1]
\State $\xi_0 \gets \Radhalf; \;\;\; \xi \gets \Radhalfd{d}$
\State $\vz^{(i)} \gets \normaldof{n}{0,1} \;$ for $1 \leq i \leq d ; \;\;\; h \gets \normaldof{d}{ 0,  \tfrac{1}{12} (t-s) }$
\State $\tb^{(i,j)} \gets \text{PairNN}_{\theta} ( ( h^{(i)}, \vz^{(i)}), ( h^{(j)}, \vz^{(j)} ))$ for $1 \leq i < j \leq d$
\State \Return $\operatorname{BF}\hspace{-0.6mm}\big(w, h, \tb, \xi_0, \xi\big)$
\end{algorithmic}
\end{algorithm}
\end{definition}\vspace*{-5mm}

\begin{figure}[H]
    \centering
    \includegraphics[width=\textwidth]{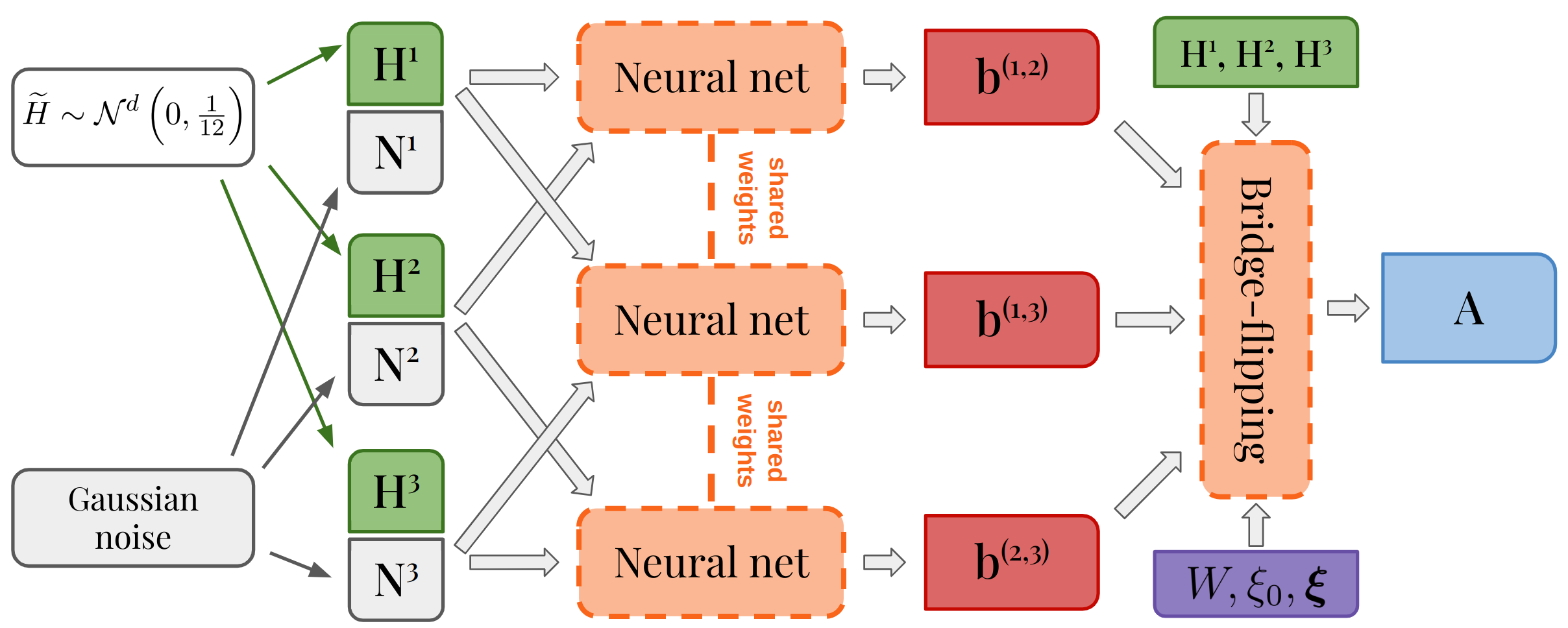}
    \caption{A schematic of the Pair-net architecture when $d=3$.}
    \label{fig:pair_net_scheme}
\end{figure}\vspace*{-2mm}
This solves the issue of permutation equivariance, but there is another requirement to be enforced:
$\bij = - b^{(j,i)}.$ Hence we want
\begin{align*}
\text{PairNN}_{\theta} \! \left( h^{(i)} \concat \vz^{(i)} \concat h^{(j)} \concat \vz^{(j)} \right) \; = \; -\text{PairNN}_{\theta} \! \left( h^{(j)} \concat \vz^{(j)} \concat h^{(i)} \concat \vz^{(i)} \right).
\end{align*}
It is possible to build this into the architecture itself, but that comes at the cost of doubling the computation time, and results in a bimodal output distribution, which is undesirable. Instead, we add an ``anti-symmetric-penalty'' to the generator loss, controlled by the hyperparameter $\lambASP$:
\begin{align*}
     L_{\mathrm{ASP}} \coloneqq \lambASP \E{\vz, \vh}{\sum_{1 \leq i < j \leq d}  \left( \text{PairNN}_{\theta} \! \left( h^{(i)} \concat \vz^{(i)} \concat h^{(j)} \concat \vz^{(j)} \right) +\text{PairNN}_{\theta} \! \left( h^{(j)} \concat \vz^{(j)} \concat h^{(i)} \concat \vz^{(i)} \right) \right)^2 }.
\end{align*}
Indeed, it can be seen empirically that this new formulation leads to substantially better permutation invariance. The simplest way to test this is to fix some $\vh \in \R^4$ and let $\vh'$ be equal to $\vh$ with the first two dimensions swapped. We then generate two batches of bridge Lévy areas $\tvb$ and $\tvb'$ conditional on $\vh$ and $\vh'$ respectively. Finally, we can permute the dimensions of $\tvb'$ in the appropriate way, i.e.
\begin{align}
     \label{eqn:swapping_b_dims}
     \tvb'^{(1,2)} \gets - \tvb'^{(1,2)} \hspace{1cm} \tvb'^{(1,3)} \longleftrightarrow \tvb'^{(2,3)} \hspace{1cm} \tvb'^{(2,4)} \longleftrightarrow  \tvb'^{(1,4)}.
\end{align}
If the generator is permutation invariant, then this should yield $\tvb \overset{d}{\approx} \tvb'$. \cref{fig:perm_inv} shows the results of such a comparison for a) a somewhat-trained Pair-net generator, and b) a slightly better trained BF generator.

\begin{figure}
     \centering
     \includegraphics[width=\textwidth]{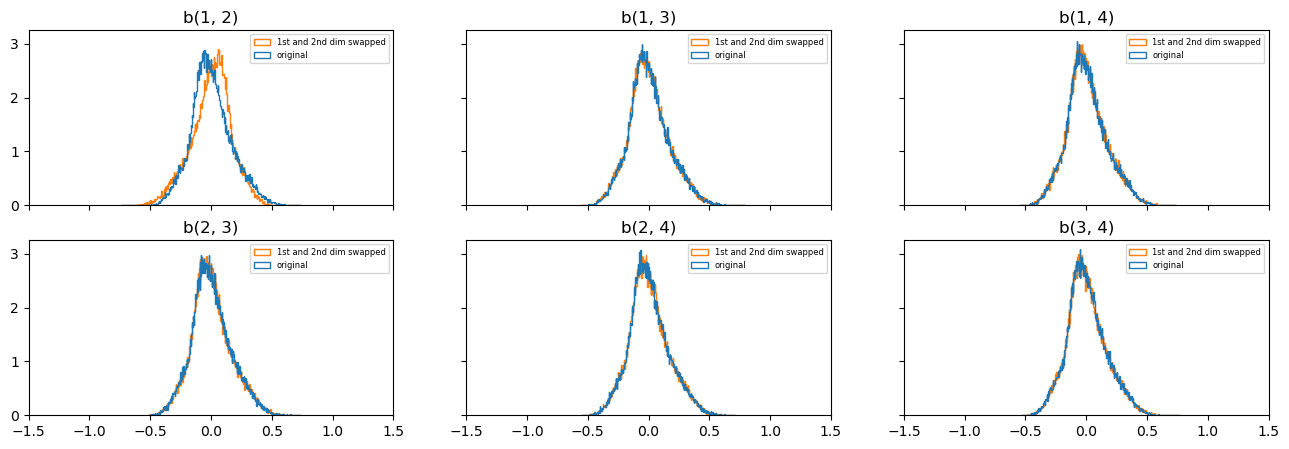} \\
     \vspace{3mm}
     \includegraphics[width=\textwidth]{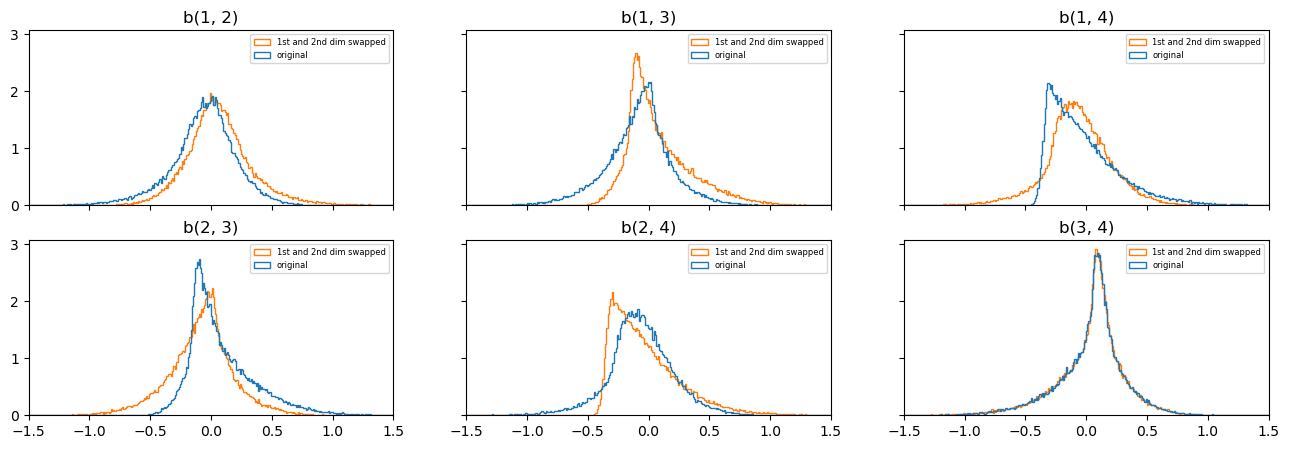}
     \caption{Probability densities estimated using samples from the generator with the first two input dimensions swapped (orange) and without swapping (blue) for Pair-net (top), and the BF generator (bottom). The blue and orange plots should overlap as best as possible.}
     \label{fig:perm_inv}
\end{figure}

Another benefit of the Pair-net architecture is that the neural net $\text{PairNN}_{\theta}$ can now be significantly smaller since it does not need to capture the relationship between all the dimensions. This comes with the downside of requiring more passes through it (as shown in \cref{fig:pair_net_scheme}, we use a single net to generate all $b$'s, each dimension of $b$ requires its own forward pass), but suitable indexing and batching can make this efficient. Indeed, using PairNet results in both a speed-up and greater accuracy (see \cref{sec:Results}).

Furthermore, we emphasize that the PairNet structure allows us to generate L\'evy area for any dimension $d'$, even though we train the model in a lower dimension $d<d'$. From now on, we will take nsz to be the total noise dimension for the complete L\'evy area generation, namely, $\text{nsz} = n \times d$ where $n$ is the latent noise dimension and $d$ denotes the Brownian dimension. In the framework of \cref{def:classicalGAN}, our generator is now $G_\theta: \R^{d}\times\R^{\text{nsz}} \rightarrow \R^d\times\R^{a}$ and the evaluation follows \cref{alg:PairNet}.
\section{Discriminator}
\label{sec:Discriminator}
One potential discriminator is the characteristic function GAN (CFGAN) approach introduced in \cite{Ansari2019,CFGAN}, which aims to learn the law of an underlying process by approximating its characteristic function. Compared to traditional GANs, we list some advantages of this approach:\smallbreak
\begin{enumerate}[label=\arabic*)]
    \item The characteristic function always exists and is uniformly bounded. \smallbreak
    \item The characteristic function fully describes the law of the random variable, hence offering good theoretical support.\smallbreak
\end{enumerate}

In this section, we introduce two different choices for the characteristic function. Firstly, we assess the distance between two distributions using the characteristic function distance.
\begin{definition}[Characteristic function distance \cite{Ansari2019}]
\label{def:cfd}
Let $\mathbb{P}_{X}$ and $\mathbb{P}_{Y}$ be the distributions of two $\mathbb{R}^d$-valued random variable $X$ and $Y$ respectively. The characteristic function distance (CFD) between  $X$ and $Y$ associated with an $\mathbb{R}^d$-valued random variable $\Lambda \sim \nu$ is given by
\begin{eqnarray*}
\text{CFD}_{\Lambda}(X, Y) = \mathbb{E}_{\Lambda \sim \nu}\Big[\norm{\Phi_{X}(\Lambda) - \Phi_Y(\Lambda) }\Big],
\end{eqnarray*}
where $ \Phi_X(\Lambda) := \E_{X\sim \mathbb{P}_X} [\exp{i\langle \Lambda, X \rangle}]$.
\end{definition}

The characteristic function distance has the following properties:\smallbreak
\begin{enumerate}[label=\arabic*)]
    \item Definiteness: If the support of $\Lambda$ is $\mathbb{R}^d$, then $\text{CFD}_{\Lambda}(X, Y) = 0$ iff $\mathbb{P}_{X} = \mathbb{P}_{Y}$. \smallbreak
    \item The distance is bounded and differentiable almost everywhere. \smallbreak
    \item For certain $\nu$, $\text{CFD}_{\Lambda}(X, Y_n) \to 0 \implies Y_n \overset{d}{\to} X$ (\cite{CFMMD},\cref{prop:weak_convergence}).\medbreak
\end{enumerate}

If we restrict $\nu$ to be a member of the family of certain well-known distributions, Gaussian for example, we can parameterize the distribution $\nu$ by learnable coefficients and it will play the role of discriminator in the GAN setting. The backpropagation on these coefficients is well-understood by techniques such as the reparameterization trick. In practice, we approximate the characteristic function by an empirical measure: if $\Lambda\sim \nu$, and $x_1,\dots,x_N$ are samples from $\mathbb{P}_X$, then we estimate $\Phi_{X}(\Lambda)$ by
\begin{equation*}
    \hat{\Phi}_{X}(\Lambda) = \frac{1}{N} \sum_{i=1}^N \exp\big( i \langle \Lambda, x_i \rangle \big).
\end{equation*}

The use of the empirical characteristic function in place of the analytical characteristic function is justified in \cref{prop:empirical_approx}. However, if $X = (W_{t}, A_t)$, i.e. the joint process of $d$-dimensional Brownian motion and the corresponding L\'evy area at any time $t$, we may obtain $\Phi_{(W_{t}, A_t)}$ analytically, see \cref{main_theorem_into}.    
\subsection{Unitary characteristic function}
\label{sec: ucf}
In this subsection, we introduce an extension of the classical characteristic function of a random variable, originally proposed in \cite{Chevyrev2016characteristic} and \cite{HangPCF}. We denote by $U_n$ the set of unitary matrices of dimension $n$, then $U_n$ is a matrix Lie group with the group operation of matrix multiplication. The Lie algebra of $U_n$, denoted by $\mathfrak{g}_n$ is the set of anti-hermitian matrices, i.e. $\mathfrak{g}_{n} :=\{A\in \mathbb{C}^{n \times n}: A^*+A=0\}$. Next, we give a definition of the unitary representation of a random variable and its unitary characteristic function.


\begin{definition}[Unitary characteristic function]
\label{def:UCF}
Let $X\sim \P_{X}$ be a $\R^d$-valued random variable. Let $n\geq 1$ be an integer. Denote by $U_n$ and $\mathfrak{g}_n$ the unitary matrix Lie group of degree $n$ and its corresponding Lie algebra. Let $M\in\LL(\R^d, \mathfrak{g}_n)$, the \textit{unitary representation function} of $X$ is given by the mapping :
\begin{equation}
\label{eq:unitary_representation}
\mathcal{U}_{M}(X) \coloneqq \exp(M(X))
\end{equation}
where $\exp$ denotes the matrix exponential. The \textit{unitary characteristic function} of $X$ is defined as a mapping from $L(\mathbb{R}^d, \mathfrak{g}_n)$ to $GL(n)$ such that
\begin{equation}
    \text{UCF}_n(X)(M) \coloneqq \mathbb{E}_{X \sim \mathbb{P}_{X}}[\mathcal{U}_{M}(X)]
\end{equation}
In practice, $\text{UCF}_n(X)(M)$ is approximated by Monte-Carlo. Let $\vX = \{X_i\}_{i=1}^N$ be $N$ samples from $\P_X$. The \textit{empirical unitary characteristic function} is given by
\begin{equation*}
    \text{EUCF}_n(\vX)(M) \coloneqq \frac{1}{N}\sum_{i=1}^N \mathcal{U}_{M}(X_i).
\end{equation*}
\end{definition}

\begin{remark}
If $n=1$, then $\mathfrak{g}_1$ is just the set of pure imaginary numbers and $M$ can be identified with an element $\lambda \in \R^d$ such that $M(x) = i\langle \lambda, x \rangle$ for all $x\in \R^d$ where $\langle \cdot, \cdot \rangle$ denotes the standard inner product. Then, $\text{UCF}_1(X)$ recovers the standard characteristic function of a $d$-dimensional random variable.
\end{remark}


Similar to \cref{def:cfd}, for any 
$n\geq1$, we can define the \textit{unitary characteristic function distance} for two random variables $X$ and $Y$ as 
\begin{equation*}
    \text{UCFD}^2_n(X,Y) = \mathbb{E}_{M \sim \mathbb{P}_{\mathcal{M}}}\Big[ \norm{ \text{UCF}_n(X)(M)-\text{UCF}_n(Y)(M)  }^2_{HS}\Big].
\end{equation*}
where $\mathbb{P}_{\mathcal{M}}$ denotes the distribution of the linear mapping $M$ and $\lVert \cdot \rVert_{HS}$ denotes the Hilbert–Schmidt norm. As was the case with the discriminator for the standard characteristic function, we parameterize the law of linear mappings $\mathbb{P}_{\mathcal{M}}$ by an empirical measure
\begin{equation*}
    \mathbb{P}_{\mathcal{M}} = \frac{1}{N} \sum_{i=1}^{N} \delta_{M_i},
\end{equation*}
where $\delta$ denotes the Dirac measure and each $M_i$ can be parametrized as learnable coefficients and optimized using gradient-based methods. For computation and optimization details of UCFD, please refer to \cite{HangPCF}.  Let $X = (W_1, A_1)$, then the benefits of UCFD compared to the standard characteristic function distance include:\smallbreak
\begin{enumerate}[label=\arabic*)]
    \item Standard uniqueness results hold as $\text{UCF}_1(X)$ recovers the standard characteristic function.\smallbreak
    \item Since $\mathfrak{g}_1$ is a subspace of $\mathfrak{g}_n$ for any $n>2$, $\mathcal{U}_M$ possesses a richer structure than $\mathbb{C}$ for any linear mapping $M$ into $\mathfrak{g}_n$. Although using $\mathfrak{g}_1$ already encodes enough information to determine the random variable, embedding the random variable into Lie algebra of a higher degree  appears to provide a more efficient way of representing the information that characterizes the random variable. Empirically, using the unitary characteristic function led to a more stable training procedure compared to the standard one.\smallbreak
\end{enumerate}

\renewcommand{\SS}{\mathcal{S}}

\subsection{Chen Training}
\label{sec:ChenTraining}

While Brownian motion over different intervals can be concatenated simply via addition, i.e.~$W_{0,t} = W_{0,s} + W_{s,t}$, concatenation of L\'evy areas requires an additional term, specified by Chen's relation. In its general version within rough path theory, Chen's relation establishes the homomorphism property of path signatures under concatenation \cite[Theorem 2.9]{LyonsBook2007}. However, we will present just the special case relating to L\'evy area.
\begin{proposition}[Chen's relation \cite{Chen57}]
\label{prop:ChenRel}
For times $0 < s < t$
{\small 
\[
A^{(i,j)}_{0,t} = A^{(i,j)}_{0,s} + A^{(i,j)}_{s,t} + \frac{1}{2}\Big( W^{(i)}_{0,s} W^{(j)}_{s,t} - W^{(j)}_{0,s} W^{(i)}_{s,t} \Big).
\]
}
\end{proposition}

We now present the main theoretical contribution of this section, which can be viewed as a partial converse of the above proposition.
\begin{theorem}\label{thm:DistrChenUniquenessGeneral}[Distributional uniqueness of L\'evy area under Chen's relation]

Suppose $\mu$ is a mean zero probability distribution on $\R^{d}\times \R^{d\times d}$, where the first marginal has finite second moment. Let $(V_i,Z_i)\stackrel{i.i.d.}{\sim}\mu$ for $i =1,2$; if it holds that
    \begin{align}\label{eq: chen combine}
        V_3 \coloneqq \tfrac{1}{\sqrt{2}} \left( V_1 + V_2 \right), \;\; Z_3 \coloneqq \tfrac{1}{2} Z_1 + \tfrac{1}{2} Z_2 + \tfrac{1}{4} \big( V_1 \otimes V_2 - V_2 \otimes V_1 \big)
    \end{align}
    is also distributed according to $\mu$, then $\mu$ is the distribution of $(W_{0,1}, A_{0,1})$ where $A$ is the L\'{e}vy area process associated with a $d$-dimensional Brownian motion $W$.
\end{theorem}

Including a finite-variance assumption on the measure $\mu$, and a Gaussian assumption on the first marginal, provides an alternative proof using Wasserstein distances (see \cref{thm:DistrChenUniquenessFiniteVar}). The preceding motivates the following procedure, which takes samples $Z_1, Z_2$ from some distribution and concatenates them using Chen's relation to produce samples that are closer in distribution to $A_{0,1}$. With abuse of notation, if $\vX\in \R^{m\times d}$ consists of $m$ samples of a $d$-dimensional random variable, we denote by $\vX^{(i)}\in\R^m$ the $i$-th coordinate of each sample. Adopting this notation, we describe the Chen-combine operation in \cref{alg:ChenComb}.\smallbreak

We also note that the proof of \cref{thm:DistrChenUniquenessFiniteVar} shows that the repeated application of Chen-combine gives convergence of order at least $\tfrac{1}{2}$ in the 2-Wasserstein metric. Faster rates can be proven assuming a suitable starting distribution, such as Davie's approximation \cite{Davie, FL}.

\begin{remark}
    Let $\tvA$ be the estimated L\'evy area given a fixed Brownian increment $\vW$. Let $\tvA_{\text{Chen}}$ be the resulting L\'evy area process of $\chencomb(\vW, \tvA)$. In \cref{thm:chen_err_bd} we show that (informally speaking) 
    \begin{align*}
        \WW_2\big(\tvA, \vA_{\mathrm{true}}\big) \leq \big( 2 + \sqrt{2}\m \big) \WW_2\big(\tvA, \tvA_{\text{Chen}}\big),
    \end{align*}
    where $\WW_2$ denotes the $2-$Wasserstein metric. So, in order to minimise $\WW_2(\tvA, \vA_{\mathrm{true}}),$ we need only minimise $\WW_2(\tvA, \tvA_{\text{Chen}}),$ which requires no ``true'' samples of L\'evy area to compute. This allows us to modify the training objective of the GAN so that it will learn the correct distribution \emph{without any access to externally supplied data}.
\end{remark} 


Assume $(\vW, \tvA) \in \R^{2m\times(d+a)}$ is $2m$ samples of $(d+a)$-dimensional random variable. By $\chencomb(\vW, \tvA)$ we mean: evenly split $(\vW, \tvA)$ into two blocks with equal size, denoted by $(\vW, \tvA)_{\text{first}},(\vW, \tvA)_{\text{second}} \in \R^{m\times (d+a)}$ , and apply the operation described in \cref{alg:ChenComb}. The output of $\chencomb$ will be an element in $\R^{m\times (d+a)}$.

\begin{algorithm}[H]
\setstretch{1.5}

\caption{Chen-combine}\label{alg:ChenComb}

\textbf{Input:} $m$ - batch size, $d$- Brownian dimension, $a$- L\'{e}vy dimension, $(\vW, \tvA)_{\text{first}} \in \R^{m\times (d+a)}$,  $(\vW, \tvA)_{\text{second}} \in \R^{m\times (d+a)}$ - Brownian increments and L\'{e}vy area samples.
\begin{algorithmic}[1]
\State $\hat{\vW} \;\gets\; \textbf{0}\in \R^{m\times d}$, $\hat{\vA} \;\gets\; \textbf{0}\in \R^{m\times a}$
\State $\vW_{\text{first}} \gets \frac{1}{\sqrt{2}} \vW_{\text{first}}$, $\vW_{\text{second}} \gets \frac{1}{\sqrt{2}} \vW_{\text{second}}$, \Comment{Brownian scaling $W_{0,\frac{1}{2}} \eqd \frac{1}{\sqrt{2}} W_{0,1}$}
\State $\tvA_{\text{first}} \gets \frac{1}{2} \tvA_{\text{first}}$, $\tvA_{\text{second}} \gets \frac{1}{2} \tvA_{\text{second}}$, \Comment{Brownian scaling $A_{0,\frac{1}{2}} \eqd \frac{1}{2} A_{0,1}$}
\For{$i \in \{1,\ldots, d\}$}
\State $\hat{\vW}^{(i)} \;\gets\; \vW_{\text{first}}^{(i)} + \vW_{\text{second}}^{(i)}$
\For{$j \in \{i+1,\ldots, d\}$}
\State $\vD \;\gets\; \frac{1}{2}\left(\vW_{\text{first}}^{(i)}\odot \vW_{\text{second}}^{(j)}-\vW_{\text{first}}^{(j)}\odot \vW_{\text{second}}^{(i)}\right)$
\State $\hat{\vA}^{(i,j)} \; \gets \; \tvA_{\text{first}}^{(i,j)} + \tvA_{\text{second}}^{(i,j)} + \vD$,\Comment{Chen's relation}

\EndFor
\EndFor
\State \Return $(\hat{\vW}, \hat{\vA})$
\end{algorithmic}
\end{algorithm}

\section{LévyGAN}\label{sec:levygan}
In this section, we incorporate the ideas presented in \cref{sec:generator,sec:Discriminator} into our tailored model, named L\'{e}vyGAN used to generate the associated L\'{e}vy area conditioned on the Brownian increments. Similar to \cref{def:classicalGAN}, we provide the definition of the proposed model as follows.

\begin{definition}[L\'{e}vyGAN]
    \label{def:LévyGAN}
    Let $\text{PairNN}_\theta$ denote a PairNet generator defined in \cref{def:pairnet}. Given a $d$-dimensional Brownian increment at $t=1$, $W\sim \normal^d(0, 1)$, for all $1\leq i<j\leq d$, we generate estimated L\'{e}vy area associated to $W$ as follows:
    \begin{align}
        \Tilde{b}^{(i,j)} & = \text{PairNN}_{\theta}((H^{(i)}, Z^{(i)}), (H^{(j)}, Z^{(j)})) \label{eq:generator-1}
        \\
        \Tilde{A}^{(i,j)} & = H^{(i)}W^{(j)} - H^{(j)}W^{(i)} + \Tilde{b}^{(i,j)} \label{eq:generator-2}
    \end{align}

    Let $\vw \in \R^{N\times d}$ be $N$ samples of Brownian increment and the associated $\tvA \in \R^{N\times a}$ generated according to \cref{eq:generator-1,eq:generator-2}, then we construct new samples using $\chencomb$ defined in \cref{alg:ChenComb}:
    \begin{equation*}
        (\vw_{\text{Chen}}, \tvA_{\text{Chen}}) \coloneqq  \chencomb(\vw, \tvA).
    \end{equation*}
    
    For $m \geq 1$, let $\mathfrak{g}_m$ be the Lie algebra of the unitary matrix group $U_m$. Recall $\text{EUCF}_m$ from in \cref{def:UCF}, and let $\mathcal{M} = \{M_i\}_{i=1}^{M},\ M_i \in L(\R^{d+a}, \mathfrak{g}_m)$ be a collection of linear mappings onto $\mathfrak{g}_m$, each of them parametrized by an element in $\R^{(d+a)\times \text{dim}(\mathfrak{g}_m)}$. $\mathcal{M}$ will play the role of the discriminator.
    Finally, the training is performed with respect to the following min-max game
     \begin{equation*}
         \min_{\theta} \max_{\mathcal{M}} \text{Loss}(\theta, \mathcal{M}; \vw),
     \end{equation*}
     where $ \text{Loss}(\theta, \mathcal{M}; \vw) \coloneqq \text{EUCFD}_m\big((\vw, \tvA), (\vw_{\text{Chen}}, \tvA_{\text{Chen}})\big).$ The training algorithm and flowchart are described in \cref{alg:LevyGAN,fig:chen-trn_scheme}.
\end{definition}

\begin{remark}
    One can interpret Chen training as a type of adaptive training, where the ``almost true'' target data $\tvA_{\text{Chen}}$ is always just sufficiently better than the generator's output, that training can progress effectively. Thus, instead of requiring large datasets of ``true'' samples, which are costly to generate and difficult to handle, we can now very efficiently generate new ``true'' data on the fly, of any desired quantity and of just the right precision. One could choose to iteratively apply Chen-combine several times, but we have observed that training is slightly faster and more efficient when only a single application of Chen combined is used in \cref{alg:LevyGAN}. This is because using two Chen-combines produces twice fewer ``true" samples and reducing the sample size leads to a poorer estimate of EUCFD.
\end{remark}\vspace*{-2mm}

\begin{algorithm}[H]

\caption{Training algorithm for L\'{e}vyGAN}\label{alg:LevyGAN}

\hspace*{\algorithmicindent}

\textbf{Input:} $d$ - Brownian dimension, $a$ - L\'{e}vy area dimension, $n$ - noise dimension, $m$ - Lie algebra degree, $M$ - number of linear mappings onto $\mathfrak{g}_m$, $\text{PairNN}_\theta$ - generator, $\mathcal{M}\in \R^{M\times (d+a) \times \text{dim}(\mathfrak{g}_m)}$ - discriminator, $\text{iter}_d$ - number of discriminator updates per generator update, \texttt{bsz} - batch size, $\eta_g$, $\eta_d$ - generator and discriminator learning rates.\smallbreak

\begin{algorithmic}[1]

\While{$\theta, \mathcal{M} \text{ not converge}$}\smallbreak

\For{$i \in (1,\dots,\text{iter}_d)$}\smallbreak

\State $\text{Sample }\vw \sim \normal^d(0, 1),\; (\vh, \vz) \sim \normal^d(0, \tfrac{1}{12}) \times \normal^{d\times n}(0, 1)$ of size $2 \, \bsz$.\smallbreak

\State $\tvb^{(i,j)} \; \gets \; \text{PairNN}_{\theta} ( ( \vh^{(i)}, \vz^{(i)}), ( \vh^{(j)}, \vz^{(j)} ))$ for $1 \leq i < j \leq d$\smallbreak

\State $\xi_0 \gets \Radhalf$,\hspace{2.5mm} $\vxi \gets \Radhalfd{d}$\smallbreak

\State $\tvA\; \gets \; \operatorname{BF}\hspace{-0.6mm}\big( \vw, \vh, \tvb, \xi_0, \vxi \big)$\smallbreak

\State $\vw_{\text{Chen}}, \; \tvA_{\text{Chen}} \; \gets \; \chencomb(\vw, \tvA)$\smallbreak

\State $\text{Loss}(\theta, \mathcal{M}; \vw) \; \gets \; \text{EUCFD}_m((\vw, \tvA), (\vw_{\text{Chen}}, \tvA_{\text{Chen}}))$\smallbreak

\State $\mathcal{M} \; \gets \; \mathcal{M} - \eta_d \cdot \nabla_\mathcal{M} (-\text{Loss}(\theta, \mathcal{M}; \vw))$\Comment{Maximize the loss}\smallbreak

\EndFor

\State $\text{Sample }\vw \sim \normal^d(0, 1),\; (\vh, \vz) \sim \normal^d(0, \tfrac{1}{12}) \times \normal^{d\times n}(0, 1)$ of size $2 \, \bsz$.\smallbreak

\State $\tvb^{(i,j)} \; \gets \; \text{PairNN}_{\theta} ( ( \vh^{(i)}, \vz^{(i)}), ( \vh^{(j)}, \vz^{(j)} ))$ for $1 \leq i < j \leq d$\smallbreak

\State $\xi_0 \gets \Radhalf$,\hspace{2.5mm} $\vxi \gets \Radhalfd{d}$\smallbreak

\State $\tvA\; \gets \; \operatorname{BF}\hspace{-0.6mm}\big( \vw, \vh, \tvb, \xi_0, \vxi \big)$\smallbreak

\State $\vw_{\text{Chen}}, \; \tvA_{\text{Chen}} \; \gets \; \chencomb(\vw, \tvA)$\smallbreak

\State $\text{Loss}(\theta, \mathcal{M}; \vw) \; \gets \; \text{EUCFD}_m((\vw, \tvA), (\vw_{\text{Chen}}, \tvA_{\text{Chen}}))$\smallbreak

\State $\theta \; \gets \; \theta - \eta_d \cdot \nabla_{\theta} \text{Loss}(\theta, \mathcal{M}; \vw)$\Comment{Minimize the loss}\smallbreak

\EndWhile

\State \Return $\text{PairNN}_{\theta}, \mathcal{M}$
\end{algorithmic}
\end{algorithm}\vspace*{-6mm}



\begin{figure}[H]
    \centering
    \includegraphics[width=\textwidth]{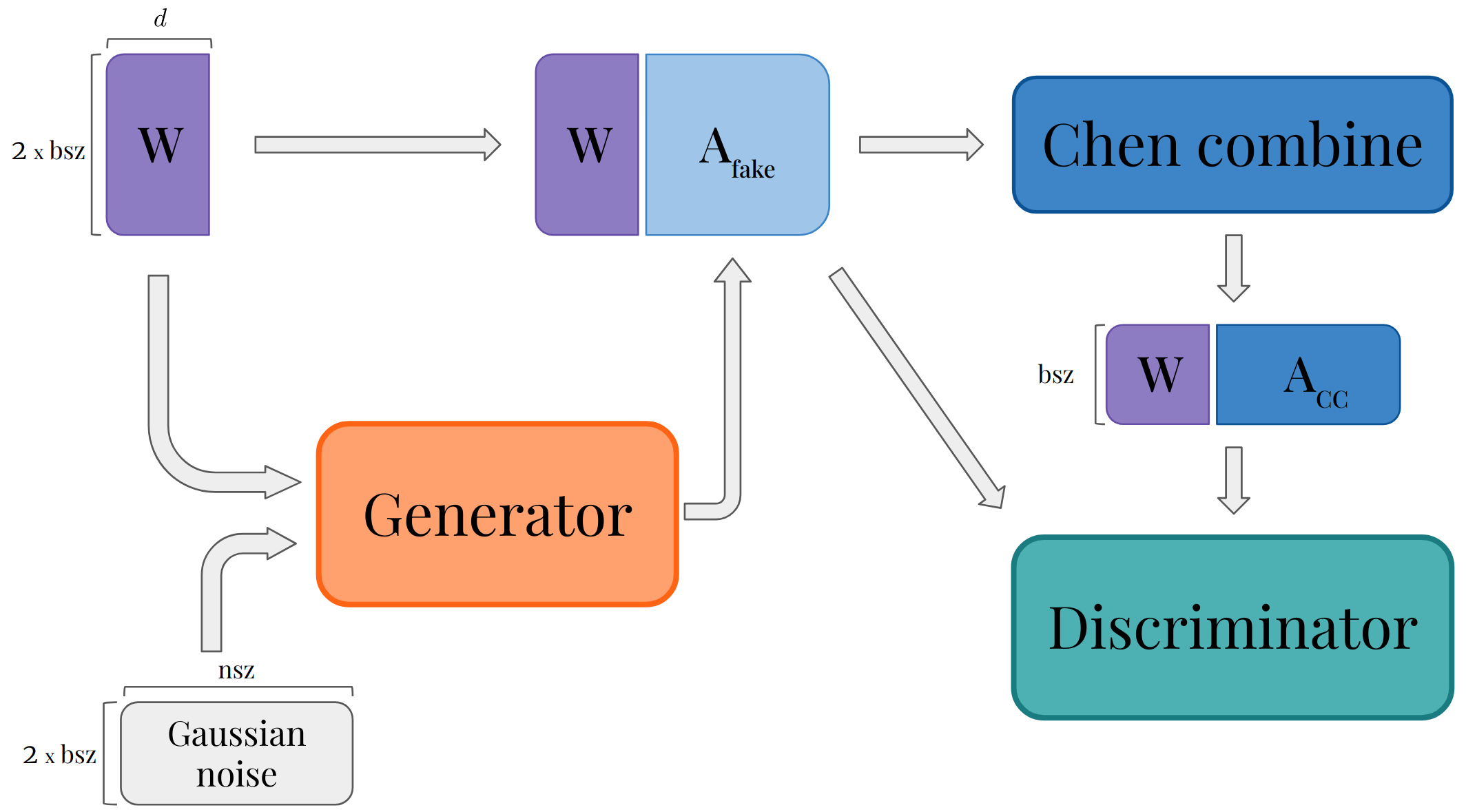}\vspace*{2mm}
    \caption{A schematic of L\'{e}vyGAN. Here $bsz$ denotes the batch dimension and we recall that $nsz$ denotes the total noise dimension, namely $nsz = n \times d$.}
    \label{fig:chen-trn_scheme}
\end{figure}
\section{Numerical Experiments}
\label{sec:Results}

We train the model in $d=4$. Note that by the architecture of the generator, the model can be used to generate Brownian L\'{e}vy area for any $d^{\m\prime}\leq d$. The model can be also used to generate L\'{e}vy area for any $d^{\m\prime} > d$, however, the performance might be deteriorated as training is not done for higher dimensions.\smallbreak

We performed the training procedure as illustrated in \cref{alg:LevyGAN}. On the generator side, we used a Feed-forward Neural Network. The activation function is chosen to be LeakyRelu function. On the discriminator side, we parameterize $128$ linear maps onto the Lie algebra of degree $3$ to mimic the empirical distribution used to compute UCFD mentioned in \cref{sec: ucf}. The total number of training iterations is set to be $2500$, where we observed the convergence on the marginal $2$-Wasserstein metric on real data. We optimize both the generator and discriminator using Stochastic Gradient Descent and Adam optimizer. We set the batch size to $2^{13}$ and the learning rate for generator/discriminator is set to be $0.001$/$0.01$ respectively. Both learning rates decay for each $500$ iteration. Finally, we set $\text{iter}_d$ to be $3$.\smallbreak

We conducted a hyperparameter grid-search (see \cref{sec:param_tuning}), evaluating the model performance according to the marginal 2-Wasserstein metric, with our optimal architecture as follows:\smallbreak
    \begin{itemize}
        \item Feed-forward Neural Network with 3 hidden layers and 16 hidden dimensions.\smallbreak
        \item LeakyRelu activation function with $\text{slope} = 0.01$.\smallbreak
        \item Gaussian noise with $n=3$.\smallbreak
    \end{itemize}

Finally, we assess the performance of our model on the generation of the coupled process for $d=2,\ 3,\ 4$, and $8$. We consider the following test metrics:\smallbreak
\begin{enumerate}[label=\arabic*)]
    \item Marginal 2-Wasserstein metric.\smallbreak
    \item Cross moment metric.\smallbreak
    \item Characteristic Function Distance, using Maximum Mean Discrepancy with different kernels.\smallbreak
    \item Empirical Unitary Characteristic Function.\smallbreak
\end{enumerate}

A detailed explanation of each test metric can be found in \cref{appendix:test metrics}. We compare with two baselines: Foster's and Davie's moment matching generator \cite{FosterThesis, foster2025levy, Davie}, and we regard the truncated Fourier series \cite{KR} of L\'{e}vy area up to an $L^2$ precision of $10^{-4}$ as ``true'' samples. Finally, we provide a numerical example for the log-Heston model using different estimators for fake L\'evy area. 
\begin{table}[H]
    \begin{center}
    \begin{tabular}{@{}ccccc@{}}
    \toprule
    \textbf{Test Metric} & L\'evyGAN  & Foster & Davie & Fourier series \\ \midrule
   Computational time (s) & $0.019$ &$0.0071$  & $0.002$ & $3.1$ \\[2pt]
   Marginal $W_2$ ($10^{-2}$) & $\mathbf{.246 \pm .013}$ & $.254\pm .010$ & $2.03\pm.013$ & $.27 \pm 0.008$\\
     \bottomrule
    \end{tabular}\vspace*{2.5mm}
    \end{center}
    \caption{Marginal distribution fitting and computational efficiency for the different generative models. The generation is done using NVIDIA Quadro RTX 8000. The marginal $W_2$ error is calculated with respect to the joint process generated by the Fourier series. Tests are performed with $2^{20}$ samples. The final column contains the results of the Fourier algorithm with 19 terms in the expansion (far more terms were used to generate ``true'' samples). This truncation has been chosen so that the performance is comparable to L\'evyGAN and Foster's method.}
\end{table}

\begin{table}[H]
\begin{center}
\begin{tabular}{@{}ccccc@{}}
\toprule
\textbf{Dim} & \textbf{Test Metrics} & L\'evyGAN  & Foster & Davie  \\ \midrule
\multirow{4}{*}{2}
 & Fourth moment &$.004\pm .002$  & $\mathbf{.002\pm .002}$& $.042\pm.001$  \\[2pt]
 & Polynomial MMD ($10^{-5}$) & $\mathbf{.341 \pm .070}$ & $.654 \pm .131$ &$.646\pm.188$ \\[2pt]
  & Gaussian MMD ($10^{-6}$) & $1.47\pm.125$& $\mathbf{1.44 \pm .128}$&$ 34.6\pm.683$ \\
  [2pt]
  & EUCFD ($10^{-2}$) & $\mathbf{1.52\pm.213}$& $1.92 \pm .113$&$ 10.1\pm.851$ \\
 \midrule
\multirow{4}{*}{3} 
 & Fourth moment & $\mathbf{.004 \pm .002}$ &$\mathbf{.004\pm.002}$  & $.043\pm.001$ \\[2pt]
 & Polynomial MMD ($10^{-5}$) & $\mathbf{2.18 \pm .568}$ & $2.30 \pm .732$ &$2.26 \pm .773$ \\[2pt]
  & Gaussian MMD ($10^{-6}$) & $1.87 \pm .002$ & $\mathbf{1.84 \pm .001}$ &$16.3\pm.001$ \\
  [2pt]
  & EUCFD ($10^{-2}$) & $\mathbf{1.88\pm.063}$& $2.03 \pm .034$&$ 18.5\pm1.11$ \\
 \midrule
 \multirow{4}{*}{4} 
 & Fourth moment & $\mathbf{.004 \pm .000}$ & $.006 \pm .002$ & $.043\pm.002$ \\[2pt]
 & Polynomial MMD ($10^{-5}$) & $\mathbf{4.04 \pm .436}$ & $4.65 \pm 1.31$ &$5.62 \pm .808$ \\[2pt]
  & Gaussian MMD ($10^{-6}$) & $\mathbf{1.90 \pm .001}$ & $\mathbf{1.90\pm.001}$ &$263\pm .003$ \\
  [2pt]
  & EUCFD ($10^{-2}$) & $\mathbf{1.92\pm.026}$& $2.03 \pm .036$&$ 17.5\pm.483$ \\
\midrule
 \multirow{4}{*}{8} 
 & Fourth moment & $\mathbf{.006 \pm .001}$ & $\mathbf{.006 \pm .002}$ & $.044\pm.000$ \\[2pt]
 & Polynomial MMD ($10^{-2}$) & $\mathbf{1.13 \pm .019}$ & $1.15 \pm .030$ &$1.31 \pm .066$ \\[2pt]
  & Gaussian MMD ($10^{-6}$) & $\mathbf{1.91 \pm .001}$ & $\mathbf{1.91\pm.000}$ &$ 1.92\pm .003$ \\
  [2pt]
  & EUCFD ($10^{-2}$) & $\mathbf{1.99\pm.002}$& $\mathbf{1.99\pm.001}$ & $ 2.05\pm.003$ \\
 \bottomrule
\end{tabular}\vspace*{2.5mm}
\end{center}
\caption{Fourth moment and MMD-based metrics across different models and Brownian dimensions. }
\end{table}

\subsection{SDE Example}
In this section, we will demonstrate how ``fake'' L\'{e}vy area can be used within SDE numerics to achieve both high order weak convergence as well as Multilevel Monte Carlo (MLMC) variance reduction.
Although the synthetic L\'{e}vy area only needs to exhibit the correct mean and covariance to give high order weak convergence, we show that the bias introduced by the MLMC estimator is negligible in practice due to the small Chen error inherent in our generative model. A secondary motivation is to compare the various L\'evy area generators and show that our GAN-based approach performs indistinguishably from previous state-of-the-art methods -- whilst taking less time to generate samples.\smallbreak


Consider the It\^o SDE from \cref{eq:SDE_1}
\begin{equation*}
    dX_t = f(X_t)\m dt + \sum_{i=1}^dg_i(X_t)\m dW_t^{(i)},\ \ X_0=x_0,\tag{\ref{eq:SDE_1} revisited}
\end{equation*}
where the solution $X$ takes values in $\R^e$. To estimate the solution to \cref{eq:SDE_1} one typically uses a discretisation scheme that generates approximate sample paths of the solution $X$. Often the objective is to approximate quantities of the form
\begin{equation}\label{eq:mc_to_estimate}
    \Ex\left[\varphi(X)\mid X_0 = x_0\right],
\end{equation}
where $\varphi$ may depend on the whole sample path $(X_t)_{t\in [0,T]}$, though commonly it is only a function of the solution $X_T$ at the terminal time $T$. To measure the error of a particular discretisation scheme, there are two standard metrics: weak and strong error. We will only evaluate the weak error for reasons discussed in \cref{sec:weak_vs_strong}. To accurately determine the error of various numerical schemes, we seek a multidimensional SDE and a quantity of the form (\ref{eq:mc_to_estimate}) which is known semi-analytically. Thankfully, such an example exists: the price of a European call option under the log-Heston model.
The stochastic volatility model is defined by the following two-dimensional SDE:
\begin{equation}
    \begin{split}
        dU_t &= \Big(r-\frac{1}{2}V_t\Big)dt+\sqrt{V_t}\m dW_t^{(1)},\ \ U_0\in\R\\
        dV_t&= \kappa(\theta-V_t)\m dt+\sigma\sqrt{V_t}\m dW_t^{(2)},\ \ V_0>0,
    \end{split}
\end{equation}
for a pair of independent Brownian motions $W^{(1)}$ and $W^{(2)}$. To ensure the volatility term $V$ remains positive, we must enforce the Feller condition $2\kappa\theta-\sigma^2>0$. The payoff of a European call option for a price process $S$ with $S := \exp(U)$ is given by
\begin{equation*}
    \varphi(S):=e^{-rT}\big(e^{U_T}-K\big)^+,
\end{equation*}
where $r$ is the discount rate, $K$ the strike price, and $T$ the maturity. For the derivation and form of the semi-analytic formula for the expected value of the above, we refer the reader to \cite{Terada} and \cite{logHestonMatlab}.\smallbreak

\subsection{Numerical Results}\label{sec:sde_results}
We compare four discretisation schemes combined with multilevel Monte-Carlo (MLMC) \cite{GilesMLMC}. We briefly recall that MLMC is based on the idea of a telescoping sum of expectations. Indeed, assume we have $L$ levels, and that $Y_l$ is an estimator for $X$, based on a discretisation scheme with step-size $h_l$, then we may write
\begin{equation*}
\Ex\Big[\varphi\big(Y_L\big)\Big]=\sum_{l=1}^L\Ex\Big[\varphi\big(Y_l\big)-\varphi\big(Y_{l-1}\big)\Big],
\end{equation*}
with $Y_0\equiv 0$. The MLMC estimator is then defined by
\begin{equation}
    \bar{\varphi}_{n_1,\dots,n_L} = \sum_{l=1}^L\hat{\varphi}_{n_l},\quad\text{where}\quad \hat{\varphi}_{n_l} = \frac{1}{n_l}\sum_{i=1}^{n_l}\left(\varphi\big(Y_{l}^{i,l}\big)-\varphi\big(Y_{l-1}^{i,l}\big)\right),
\end{equation}
where $Y_l^{i,l}$ is the $i$\textsuperscript{th} sample of the estimator $Y_l$ used on level $l$, and $Y_{l-1}^{i,l}$ is the $i$\textsuperscript{th} sample of the estimator $Y_{l-1}$ that is used on level $l$.  It is important to note that the pairs $(Y_{l}^{i,l},Y_{l-1}^{i,l})$ are coupled: the underlying Brownian path for each member of the pair is the same. In our case, the path on the lower level will be coarse (i.e.~a large step size) and the higher level will be fine (i.e.~a small step size). The standard condition used to ensure convergence of the telescoping sum of expectations is given by
\begin{equation*}
\Ex\Big[\varphi\big(Y_{l}^{\m\cdot,\m l}\big)\Big]=\Ex\Big[\varphi\big(Y_{l}^{\m\cdot,\m l+1}\big)\Big].
\end{equation*}
However, when incorporating a fake L\'evy area term, our coupling at each level is defined as follows.\smallbreak
\begin{enumerate}[label=\arabic*), leftmargin=3em]
    \item The Brownian increments for the fine path $Y_{l}^{\m\cdot,\m l}$ are generated with step size $h_{l}$, with the increments of fake L\'evy area generated using some estimator $\tA^{h_l}$.\smallbreak
    \item The Brownian increments on the coarse path $Y_{l-1}^{\cdot,l}$ are computed by pairwise summing the increments of the fine path. The fake L\'evy area used on the coarse path is computed using one iteration of Chen's identity applied to the increments and areas of the fine path.\smallbreak
\end{enumerate}

This scheme however introduces a bias; namely, the distribution of the fake L\'evy area used for the fine path at level $l$ will not be the same as the distribution of L\'evy area used for the coarse path at level $l+1$. We may write the effect of this by amending the telescoping expectation to be
{\small
\begin{equation}\label{eq: telescoping_bias}
\begin{split}
    \sum_{l=1}^L\Ex\Big[\varphi\big(Y_l(\tA^{h_l})\big)-\varphi\big(Y_{l-1}(\tA^{h_{l}}_{\text{CC}})\big)\Big]=&\underbrace{\sum_{l=1}^L\Ex\Big[\varphi\big(Y_l(\tA^{h_l})\big)-\varphi\big(Y_{l-1}(\tA^{h_{l-1}})\big)\Big]}_{\text{desired telescoping expectation}}\\
    &+\underbrace{\sum_{l=1}^L\Ex\Big[\varphi\big(Y_{l-1}(\tA^{h_{l-1}})\big)-\varphi\big(Y_{l-1}(\tA^{h_l}_{\text{CC}})\big)\Big]}_{\text{bias term}}
\end{split}
\end{equation}}
where we have emphasised the dependence of the fine level on the fake area $\tA^{h_l}$ and dependence of the coarse level of one Chen iteration of this fake area, denoted by $\tA^{h_l}_{\text{CC}}$. The bias introduced is exactly the second sum. We aim to show empirically that this sum is small in comparison to the size of the weak error due to the SDE discretisation. In order to minimise this sum, the distribution of an estimator $\tA$ must be as close as possible to the distribution $\chencomb(\tA)$; exactly the criterion used to train our generator.\smallbreak

The four numerical schemes in the comparison are: no-area Milstein, antithetic Milstein \cite{GilesSzpruch}, the Strang splitting method, and a ``Strang'' log-ODE method. Only the final method incorporates the fake L\'evy area. For details on the schemes see \cref{sec:discretisation_schemes}. The first two schemes were included to demonstrate that the rate of variance reduction is comparable to two popular schemes, while the weak error rate of the Strang log-ODE method is (conjecture to be) $O(h^2)$ where the other methods achieve a weak error rate of $O(h)$. The numerical simulations were performed with a constant time-step $h_l$ on each level satisfying $h_l=~\tfrac{1}{2}h_{l-1}$. On the coarsest level we use the timestep $h_0=\tfrac{1}{2}$ for the Milstein methods and $h_0=1$ for the Strang methods; this results in the variance on this level being approximately equal across the three methods. The number of sample paths on each level satisfies $n_l=\frac{1}{2}n_{l-1}$ with $n_0=2^{27}$, so the computational effort on each level is approximately constant. We repeat the experiment forty times and report the average result. We fix the log-Heston model parameters to be $T=1$, $r=0.1$, $K=20$, $\kappa=2$, $\theta = 0.1$, $\sigma = 0.5$, $U_0=\log(20)$, and $V_0=0.4$.

The following plots report the multilevel variance defined by $\text{Var}\left[\varphi\big(Y_{l}^{\cdot,l}\big)-\varphi\big(Y_{l-1}^{\cdot,l}\big)\right],$ and the empirical error given by $\big\vert\bar{\varphi}_{n_1,\dots,n_l}-P_{\text{true}}\big\vert$, where $P_{\text{true}}$ is the true price of the call option under the log-Heston model.\vspace*{-1mm}
\begin{figure}[H]
\centering
    \begin{subfigure}{.48\textwidth}
    \centering
    \includegraphics[width=1\linewidth]{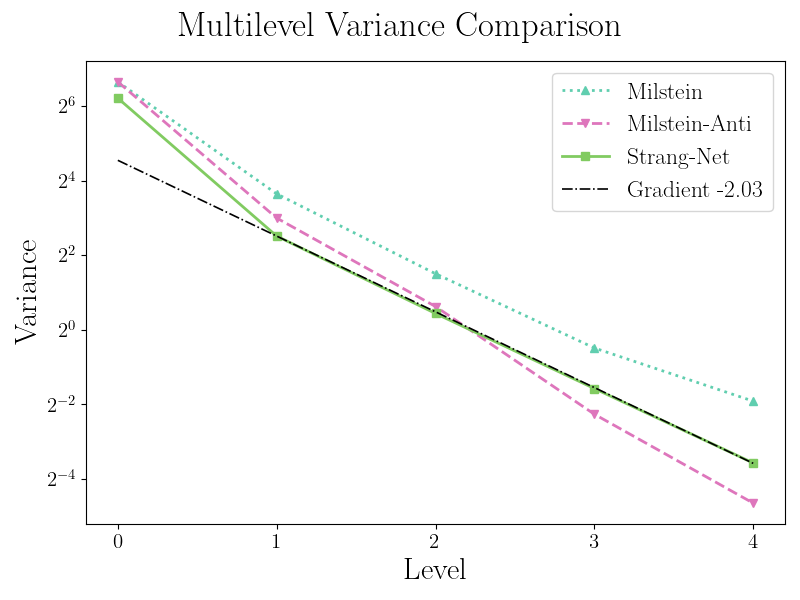}
    \caption{Multilevel variance}
    \label{fig:fake_vs_no_fake_variance}
    \end{subfigure}%
\begin{subfigure}{.48\textwidth}
\centering
\includegraphics[width=1\linewidth]{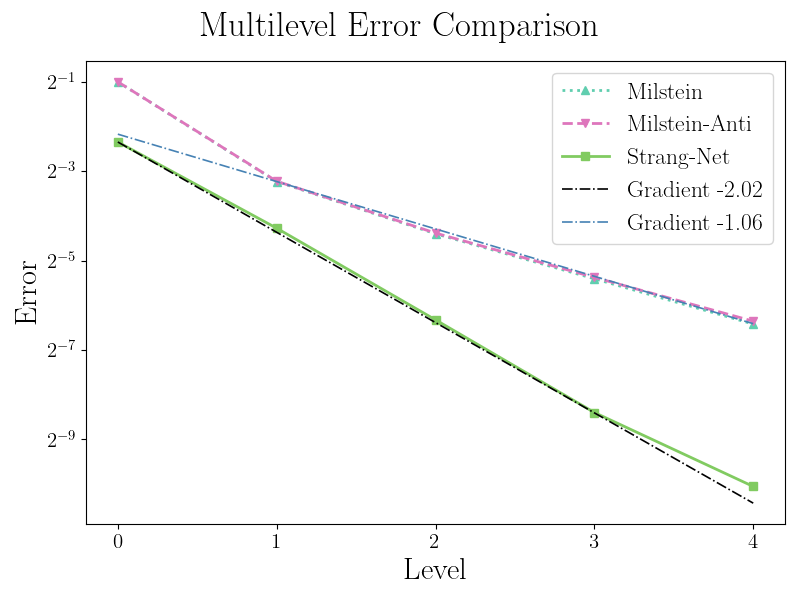}
\caption{Empirical error}
\label{fig:fake_vs_no_fake_error}
\end{subfigure}\vspace*{4mm}
\begin{subfigure}{.48\textwidth}
    \centering
    \includegraphics[width=1\linewidth]{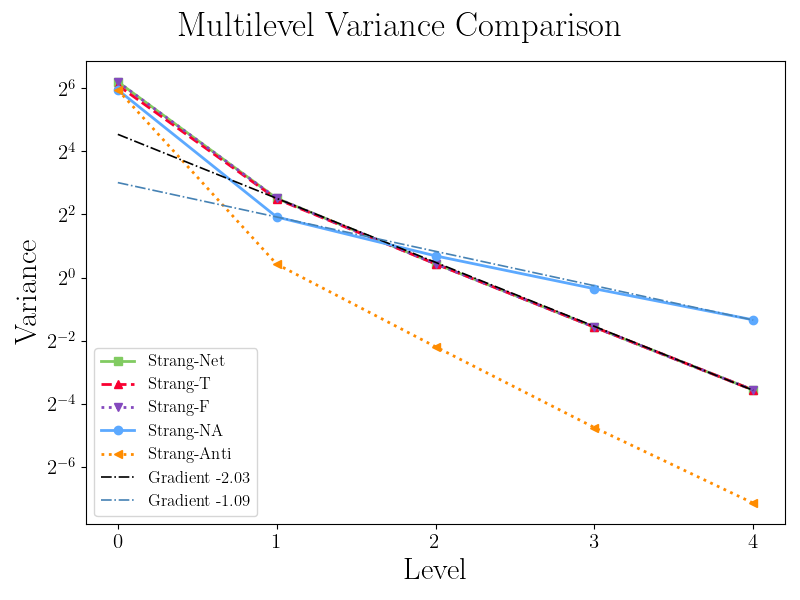}
    \caption{Multilevel variance}
    \label{fig:fakes_variance}
    \end{subfigure}
\begin{subfigure}{.48\textwidth}
\centering
\includegraphics[width=1\linewidth]{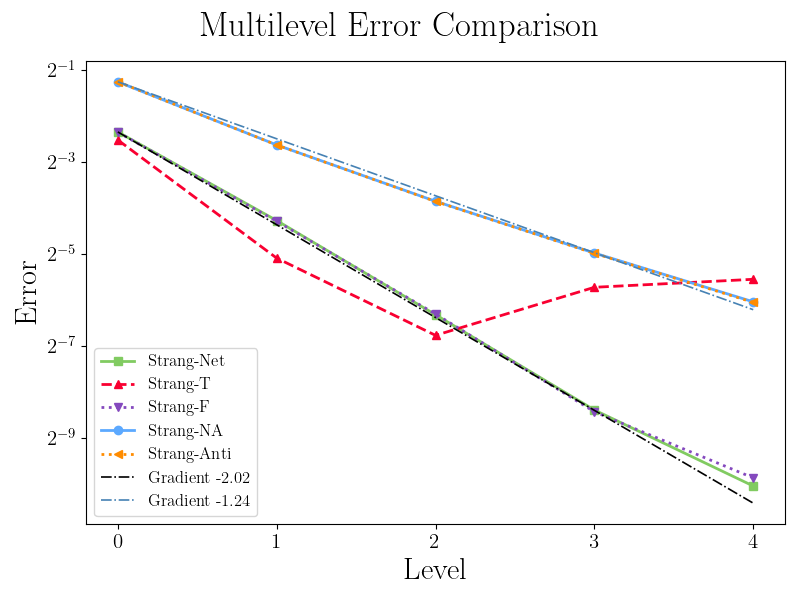}
\caption{Empirical error}
\label{fig:fakes_error}
\end{subfigure}
\caption{Plots of multilevel variance and empirical error. The top pair of plots compare the Milstein scheme without area, Milstein antithetic and Strang log-ODE scheme with fake L\'evy area from our generator (labelled ``Strang-Net"). The bottom pair compare the Strang log-ODE scheme using three different fake L\'evy areas.
``Strang-T'' indicates that the fake L\'evy area is an independent Rademacher random variable with the correct variance (the same random variable appearing in the Talay scheme), and ``Strang-F'' denotes Foster's approximation. The ``Strang-NA'' line is the usual Strang splitting method with ``Strang-Anti'' being the antithetic version of this scheme.}
\label{fig:fake_vs_no_fake}
\end{figure}\vspace*{-2mm}

In \cref{fig:fake_vs_no_fake_variance} we see that the multilevel level variance of the Strang log-ODE method decreases at an approximate rate of $O(h^{2})$. The rate for the Milstein antithetic scheme appears slightly higher, while the variance reduction rate for the standard Milstein method is clearly lower. As expected, the weak convergence rate of both the Milstein and Milstein antithetic schemes is of order $O(h)$, while the weak rate for the Strang log-ODE with fake L\'evy area is approximately $O(h^{2})$. It was conjectured in \cite{foster2024high} that the Strang log-ODE method should attain this weak convergence rate, and the experiments corroborate this hypothesis.

When using a fake L\'evy area in MLMC, a key factor for the performance of the scheme is how close the distribution of a Chen combined sample of Brownian motion and fake L\'evy area is to the distribution before performing the combine operation. Since we employed Chen training, our model succeeds in matching the distributions well enough to match and even outperform Foster's method. In \cref{table:bias}, we record the bias introduced at each level, as in \cref{eq: telescoping_bias}, for each fake L\'evy area.\vspace*{1mm}
\begin{table}[H]
    \begin{center}
    \begin{tabular}{@{}cccccc@{}}
    \toprule
    \textbf{Level} & $0$  & $1$ & $2$  & $3$\\ \midrule
   Strang-Net ($10^{-3}$) &$ -0.402$& $0.313$& $-0.133$&$-0.652$ \\[2pt]
   Strang-F ($10^{-3}$)& $0.361$& $0.848$& $0.0714$& $1.31$&\\[2pt]
   Strang-T ($10^{-3}$)& $-17.0$& $4.33$& $-1.79$&$ -1.10$\\
     \bottomrule
    \end{tabular}\vspace*{2.5mm}
    \end{center}
    \caption{Approximate bias introduced per level by the use of different fake L\'evy areas. We use $2^{29}$ samples paths on each level.}\label{table:bias}
\end{table}\vspace*{-5mm}

For L\'evyGAN and Foster's method, the bias introduced on each level of the telescoping sum is of order $2^{-11}$; this is far smaller than the weak error seen in \cref{fig:fakes_error}. However, it is possible that the accumulated bias may then be of order $2^{-9}$ on the finest level, which may account in part for the slight deviation from the line at level $4$. It is also clear that it is not enough for the fake L\'evy area simply to match the mean and variance of true L\'evy area, as demonstrated by the poor performance of the ``Strang-T'' method. However, we do note here that a scheme matching the conditional variance of L\'evy area given a Brownian increment performed similarly to Foster's method in previous experiments. We may also see from the performance of the Strang splitting method that, without the fake L\'evy area terms, one achieves only a weak order convergence rate of $O(h)$. It is interesting to note however, that \cref{fig:fakes_variance} indicates that the fake L\'evy area need only match the mean and variance of true L\'evy area to obtain improved variance reduction at each level. Even the ``Strang T'' method has the same variance reduction rate as the more sophisticated techniques despite having poor empirical error.\smallbreak

In practice, one usually wishes to obtain some target root mean-squared error (RMSE) with minimal computational cost. In this setting, for two numerical schemes with variance reduction $O(h^\beta)$ with $\beta>1$, the scheme with higher order weak convergence is not necessarily the preferred one. By the complexity theorem of Giles \cite[Theorem 3.1]{GilesMLMC}, the optimal number of sample paths on each level should be asymptotically proportional to $O(h_l^{(\beta + 1)/2})$. As such, the computational effort should be expended mostly on the coarse levels in the regime $\beta>1$, driving one towards discretisations that are computationally cheap on the lower levels. Since it is difficult measure the computational complexity of the Strang log-ODE scheme with fake L\'evy area produced by a generative model, we use the following approach. We implement the algorithm of Giles \cite[Section 5]{GilesMLMC} for the Milstein scheme and Strang log-ODE scheme and compare the average time taken to achieve a selection of target RMSEs between $0.1$ and $0.0025$.\vspace*{1mm}
\begin{table}[H]
    \begin{center}
    \centering
    \begin{tabular}{@{}cccccccc@{}}
    \toprule
    \textbf{RMSE} & $0.1$  & $0.0441$ & $0.0129$  & $0.0086$ &$ 0.0057$& $0.0038$ &$ 0.0025$ \\ \midrule
   Milstein (s) &$ \textbf{0.0097}$& $0.0256$& $0.376$&$ 1.03$&$ 2.86$& $8.63$ & $23.6$ \\[2pt]
   L\'evyGAN (s) & $0.0102$& $\textbf{0.0128}$& $\textbf{0.142}$& $\textbf{0.311}$&$ \textbf{0.806}$& $\textbf{2.25}$& $\textbf{5.83}$\\
     \bottomrule
    \end{tabular}
    \end{center}\vspace*{2.5mm}
    \caption{The average time taken across $25$ runs for the Milstein and Strang log-ODE methods to attain a target RMSE. We use the algorithm of \cite[Section 5]{GilesMLMC} to determine the number of samples on each level and the stopping condition on the number of levels used on each run. All random variables are generated in \texttt{torch} on \texttt{GPU}, with the numerical schemes implemented in \texttt{numpy} on \texttt{CPU}.}
\end{table}
\begin{remark}
  We see from \cref{fig:fakes_variance} that the the Strang-antithetic scheme achieves the highest order variance reduction with a first order weak error rate. As noted above, both factors play a role in the overall time required to achieve a desired RMSE and it is possible that the higher order variance reduction of the Strang-antithetic scheme may outperform the high order weak error of the L\'evyGAN based approach. However, it was recently shown in \cite{iguchi2025antitheticmultilevelmethodselliptic} that the antithetic method can be combined with weak estimators of L\'evy area to achieve high order variance reduction. We believe that combining their approach with ours would result in a scheme with both high order variance reduction and weak error rate. 
\end{remark}

We conclude this section by reiterating that the use of extra random variables to attain higher order weak convergence has become a popular technique, see for example \cite{Talay, NV, NN}. But, to the best of our knowledge, it has not yet been observed that the use of fake L\'evy area combined with standard multilevel Monte-Carlo can also achieve high order weak convergence and variance reduction.
\section{Generating other integrals of Brownian motion}\label{sec: SST}

To demonstrate the wider applicability of the Chen-training paradigm, we turn our attention to generating the integral $C_{s,t} = \int_s^t W_{s,r}^2 \, dr$ where $W$ is a 1-dimensional Brownian motion. While a detailed discussion of this integral is beyond the scope of this paper, we note that numerical methods for scalar noise SDEs can achieve second order strong convergence if this integral is provided alongside Brownian increments and space-time L\'{e}vy areas (see \cite{foster2020b, tang2019sst} for further details).

Analogous to the procedure outlined in this article for the generation of space-space L\'evy area, we may attempt to train a network to sample from the distribution $\P_{C_{0,1}\mid W_{0,1}=w,\ H_{0,1}=h}$, where $H$ is the space-time L\'evy area defined in \cref{def:BBridge}. Due to Brownian scaling, the distribution of $C_{0,1}$ satisfies the following scale invariance and modified Chen's relation.
\begin{proposition}[Scaling and Chen's relation for $C$]
Let $W$ and $C$ be defined as above and $0 \leq s < t$. Then
\[
C_{s,t} \eqd (t-s)^2 C_{0,1} \quad \text{ and } \quad C_{0,1} = C_{0, \frac{1}{2}} + C_{\frac{1}{2}, 1} + W_{0,\frac{1}{2}} \Big( \tfrac{W_{0,1}}{2}  + H_{\frac{1}{2}, 1} \Big).
\]
\end{proposition}

We note that, since $C_{0,1}$ is 1-dimensional and non-negative, neither Pair-Net nor bridge-flipping are required. We can then train a feed-forward neural network with $2$ hidden layers; $16$ hidden dimensions; Gaussian noise of dimension $3$ in addition to $w$ and $h$; ReLU activation function; and the absolute value applied to the output to maintain positivity. We then train using the Chen relation described in the preceding. To compare our output, generated ``true samples'' for fixed pairs $(w,h)$ by taking fine discretisation of $\int_s^t W_{s,r}^2 \, dr$ using the Diffrax library \cite{kidger2021on}.\smallbreak

Our method achieved a $2$-Wasserstein error of $3.27\times 10^{-5}$. Just as for space-space L\'{e}vy area, we can generate a Gaussian variable with the correct conditional mean and variance (see \cite{foster2020b}). However, this achieved a $2$-Wassertein error of $7.67\times 10^{-4}$; twenty times higher than our error.
This experiment demonstrates that the Chen-training approach is not only limited to L\'evy area, but applicable to other integrals of Brownian motion that have relevance in the numerical SDE simulation.

\section{Conclusion and Open Directions}
While stochastic analysis techniques are often used in generative deep learning, this article appears to be one of the first examples where deep learning methodology has provided meaningful results in an application to stochastic analysis. Indeed, we have demonstrated a proof of concept that the techniques used in L\'evyGAN have a place in the field of numerical solutions to SDEs. We remark here though, that careful consideration of the domain-specific analytical properties was required. In particular, regardless of the network size or architecture, L\'evyGAN in its initial form was an order of magnitude less accurate without the inclusion of both bridge-flipping and Pair-net.\smallbreak

One open direction for future research is a careful analysis of the conditions required by the fake L\'evy area in order to achieve optimal convergence rates in multilevel Monte Carlo. For example, we expect that fake L\'evy area could be incorporated into antithetic MLMC schemes, such as \cite{iguchi2025antitheticmultilevelmethodselliptic}.

An application of particular interest would be a GAN-based adaptive SDE solver. That is, a method that first generates a coarsely discretised path, before checking whether the step-size of the solver should be reduced. Such functionality is desirable for use in Neural SDEs \cite{Li2020, Kidger2021b, Kidger2021} and Logsig-RNN generators \cite{ni2021sig}, which are both powerful methods for modelling noisy time series data. In the context of L\'evy area generation, this would require the ability to generate L\'evy area and Brownian increments over two half intervals given the L\'evy area and Brownian increment over the larger interval. One approach would be to use the analytical characteristic function given in \cite{Geng2012} which provides the joint characteristic function evaluated at multiple time points. However, we expect the training time to be rather long, since the evaluation of the characteristic function involves solving a recursive system of matrix Ricatti equations in addition to a system of independent linear matrix ODEs of order one.\smallbreak

Finally, it is possible to extend the Chen-training approach. This might take several forms: one might derive a Chen type relation for higher order terms in the polynomial expansion of Brownian motion (e.g. for $H$ and $b$) as in \cref{sec: SST}; use the ordinary Chen's relation for the generation of higher order terms in the log-signature of Brownian motion; generating L\'evy areas for certain L\'evy processes.\smallbreak

For the third application, our approach may be applicable to the generation of L\'evy areas of $\alpha$-stable L\'evy processes, where moment matching approaches are not possible, since these processes have unbounded variance for $\alpha<2$. Indeed, the $\alpha$-stable L\'evy process $X_t^\alpha$ satisfies the scaling property $X_t^\alpha=t^{1/\alpha}X_1^\alpha$ in addition to independent and stationary increments. Combining these properties together, the distribution of its L\'evy area (defined using It\^o integration with jumps) should also be invariant under a suitably re-scaled version of Chen's relation. One could then attempt to train a network under this Chen relation, analogous to our approach for Brownian motion.
\section*{Acknowledgements}
The authors are grateful to Veronika Chronholm and Kanakira {Terada} for providing code which was then adapted for the numerical SDE tests in \cref{sec:sde_results}. HN and JT would like to thank Terry Lyons and Hang Lou for useful discussions.

\printbibliography[heading=bibintoc,title={References}]
\clearpage


\appendix

\section{Some Properties of the CFD}
This short subsection summarises two key properties of the distance $\text{CFD}_\Lambda$. Namely, the following proposition shows that the distance between the empirical characteristic function and true characteristic function converges to zero almost surely as the number of observations tends to infinity. The second result demonstrates that by carefully choosing $\Lambda\sim \nu$, convergence in $\text{CFD}_\Lambda$ implies weak convergence.
\begin{proposition}\label{prop:empirical_approx}
    Let $\{X_i\}_{i=1}^\infty$ be a collection of i.i.d. $\R^n$ valued random variables, then
    \begin{equation*}
        \lim_{n\to\infty} \mathbb{E}_{\Lambda \sim \nu} \Big[\big\vert\hat{\Phi}_{X}^n(\Lambda) - \Phi_X(\Lambda)\big\vert \Big]\to 0 \,\,\text{ almost-surely}.
    \end{equation*}
\end{proposition}
\begin{proof}
    By \cite{ECF,ECF2,ECFmultidim}, there exists a sequence of real numbers $\{T_n\}\uparrow \infty$ such that almost-surely for every $\varepsilon>0$ there exists an $N$ such that for every $n\geq N$
    \begin{equation*}
        \sup_{\abs{\Lambda}\leq T_n}\abs{\hat{\Phi}_{X}^n(\Lambda) - \Phi_X(\Lambda)}<\varepsilon.
    \end{equation*}
    It follows almost-surely for every $\varepsilon>0$ that
    \begin{align*}
        \mathbb{E}_{\Lambda \sim \nu} \Big[\big\vert\hat{\Phi}_{X}^n(\Lambda) - \Phi_X(\Lambda)\big\vert \Big]&= \int\big\vert\hat{\Phi}_{X}^n(\Lambda) - \Phi_X(\Lambda)\big\vert d\nu(\Lambda)\\
        &\leq \int_{\abs{\Lambda}\leq T_n}\big\vert\hat{\Phi}_{X}^n(\Lambda) - \Phi_X(\Lambda)\big\vert d\nu(\Lambda)+2\nu(\{\abs{\Lambda}\geq T_n\})\\
        &\leq \varepsilon\nu(\{\abs{\Lambda}\leq T_n\})+2\nu(\{\abs{\Lambda}\geq T_n\})\\
        &\to \varepsilon.
    \end{align*}
\end{proof}
\begin{proposition}\label{prop:weak_convergence}
    Let $\nu$ be the Cauchy distribution on $\R^n$ with location parameter $0$ scale parameter $\gamma$ and independent coordinates, then $\text{CFD}_{\Lambda\sim \mu}$ metrizes the topology of weak convergence on $\PP(\R^n)$.
\end{proposition}
\begin{proof}
    We consider the case $n=1$; the general case is a straightforward extension. If the law of a random variable $T$ is that of a Cauchy distribution with location parameter $0$ and scale parameter $\gamma$, then it is well known that its characteristic function is given by
    \begin{equation*}
        \Phi_T(t)=e^{-\gamma|t|},
    \end{equation*}
    and has a density given by
    \begin{equation*}
        f_T(x) = \frac{1}{2\pi}\int_\R e^{-itx}\Phi_T(t)\m dt =\frac{1}{\pi\gamma\Big(1+\big(\frac{x}{\gamma}\big)^2\Big)}.
    \end{equation*}
    It is clear that both $\Phi_T$ and $f_T$ are in $L^1(\R)$, since
    \begin{equation*}
        \int_\R \Phi_T(t)\m dt = \frac{2}{\gamma},\ \text{and }\int_\R f_T(x)\m dx = 1.
    \end{equation*}
    We also observe that
    \begin{equation*}
        \int_\R\frac{1}{f_T(x)(1+\abs{x}^4)}\m dx\leq 2\pi\gamma\int_0^\infty \frac{\left(1+\frac{\abs{x}}{\gamma}\right)^2}{(1+\abs{x})^4}\m dx<\infty.
    \end{equation*}
    An application of \cite{CFMMD}, the uniform boundedness of characteristic functions in $L^\infty(\mu)$, and a standard interpolation argument yields that $\text{CFD}_{\Lambda\sim \mu}$ metrizes weak convergence.
\end{proof}
\subsection{Proof of the Second-order polynomial Expansion}
\label{sec:BBLAdistn_proof}


\begin{proof}
The Browninan bridge $B$ and increment $W_{s,t}$ are independent, hence so are $\{H_{s,u}\}_{u \in [s,t]}$ and $\{b_{s,u}\}_{u \in [s,t]}.$ This is because $B_u = W_{s,u} - \frac{u-s}{t-s} W_{s,t},$ and $W_{s,t}$ are jointly Gaussian and $\cov(B_u, W_{s,t}) = (u-s) - \frac{u-s}{t-s}(t-s) = 0$ for all $u\in [s,t]$, so they are independent.

To prove the second assertion, recall that an integral of a Gaussian process is Gaussian. Therefore, $H$ is Gaussian. Since $\Hi$ only depends on $\Wi$ and $\Wi$ and $\Wj$ are independent for $j \not= i$ (by definition), the coordinates of $H$ are independent.
Since $H_{s,t} \eqd H_{0, t-s}$ we only need to determine the mean and variance of $H_t.$
\[
\E\big[\Hi_t\big] = \E\bigg[\frac{1}{t} \int_0^t W_{0,u} - \frac{u}{t}W_{t} \, du\bigg] = \frac{1}{t} \int_0^t \E\bigg[W_{0,u} - \frac{u}{t}W_{t}\bigg] \, du = 0.
\]
Integration by parts gives
\begin{align*}
    t\Wi_t &= \int_0^t s \, d\Wi_s + \int_0^t \Wi_s \, ds \; \implies \; \frac{1}{t} \int_0^t \Wi_s \, ds = \Wi_t - \frac{1}{t}\int_0^t s \, d\Wi_s \\[2mm]
    \Hi_t &= \frac{1}{t} \int_0^t \Wi_s - \frac{s}{t} \Wi_t \, ds = \frac{1}{t} \int_0^t \Wi_s \, ds - \frac{1}{2} \Wi_t = \frac{1}{2} \Wi_t - \frac{1}{t} \int_0^t s \, d\Wi_s.
\end{align*}

Since $\E{(W_t)^2} = t$ and $H_t \, \perp \, W_t$ we obtain
\begin{align*}
\E{\left(\Hi_t \right)^2} &= \E{\left(\Hi_t \right)^2} + \frac{1}{4} \E{\left(\Wi_t \right)^2} - \frac{t}{4} \\
&= \E\Bigg[\bigg(\frac{1}{2} \Wi_t - \Hi_t \bigg)^2\Bigg] - \frac{t}{4} \;=\; \E\Bigg[\bigg( \frac{1}{t} \int_0^t s \, d\Wi_s \bigg)^2\Bigg] - \frac{t}{4} \\
&= \frac{1}{t^2} \E\bigg[\bigg\langle \int_0^t s \, d\Wi_s \bigg\rangle\bigg] - \frac{t}{4} \;=\; \frac{1}{t^2} \E{ \int_0^t s^2 \, d \langle \Wi_s \rangle} - \frac{t}{4} \\
& = \frac{1}{t^2} \E\bigg[\int_0^t s^2 \, ds\bigg] - \frac{t}{4} \;=\; \frac{1}{12} t.
\end{align*}
\end{proof}

\begin{proof}
We will prove the decomposition for $A_{0,1}$, which can be extended to the general case by scaling. If $B$ is the Brownian bridge on $[0,1]$, then for $t \in [0,1] \;$ $W_t = t W_1 + B_t$, so
\begin{align*}
\Aijst &= \int_0^1 \Wi_t \, d\Wj_t - \frac{1}{2} \Wi_1 \Wj_1 \\
&= \int_0^1 \left(t \Wi_1 + \Bi_t \right) \, d\! \left(t \Wj_1 + \Bj_t \right) - \frac{1}{2} \Wi_1 \Wj_1 \\
&= \frac{1}{2}\Wi_1 \Wj_1 + \Wj_1 \int_0^1 \Bi_t \, dt + \Wi_1 \int_0^1 t \, d\Bj_t + \int_0^1 \Bi_t \, d\Bj_t - \frac{1}{2} \Wi_1 \Wj_1 \\
&=  \Wj_1 \Hi_1 + \Wi_1 \int_0^1 t \, d\Bj_t + \bij_{0,1}
\end{align*}
we use stochastic integration by parts and the fact that $B_1 = 0$ to get
\[
t\Bj_t = \int_0^t s \, d\Bj_s + \int_0^t \Bj_s \, ds \; \implies \; \int_0^1 t \, d\Bj_t = - \int_0^1 \Bj_t \, dt = - \Hj_1.
\]
Plugging this into the above equality gives $A^{(i,j)}_{s,t} = H^{(i)}_{s,t} W^{(j)}_{s,t} - W^{(i)}_{s,t} H^{(j)}_{s,t} + b^{(i,j)}_{s,t}$ as required.
\end{proof}

\section{Proof of the Chen-Uniqueness Theorem}
\label{sec:chen_uniq_proofs}
Here we present the proofs of our main theoretical contributions. Firstly, we prove that the joint law of a Brownian increment and its L\'evy area is the unique distribution that is invariant under an iteration of the Chen-combine operation.
\begin{theorem}[Distributional uniqueness of L\'evy area under Chen's relation]\label{thm:DistrChenUniquenessGeneral_sm}

Suppose $\mu$ is a mean zero probability distribution on $\R^{d}\times \R^{d\times d}$, where the first marginal has finite second moment. Let $(V_i,Z_i)\stackrel{i.i.d.}{\sim}\mu$ for $i =1,2$; if it holds that
    \begin{equation}\label{eq: chen combine sm}
        V_3 \coloneqq \tfrac{1}{\sqrt{2}} \left( V_1 + V_2 \right), \;\; Z_3 \coloneqq \tfrac{1}{2} Z_1 + \tfrac{1}{2} Z_2 + \tfrac{1}{4} \left( V_1 \otimes V_2 - V_2 \otimes V_1 \right)
    \end{equation}
    is also distributed according to $\mu$, then $\mu$ is the distribution of $(W_{0,1}, A_{0,1})$ where $A$ is the L\'{e}vy area process associated with a $d$-dimensional Brownian motion $W$.
\end{theorem}
\begin{proof}
    Let $(\Omega, \FF, \P)$ be a probability space carrying for each $N\in\N$ a sequence of random variables~$(V_i^{N,N},Z_i^{N,N})_{i=1}^{2^N}\stackrel{\text{i.i.d.}}{\sim}\mu$. For each $N$ define via backward recursion the sequences $\big(V_i^{k,N},Z_i^{k,N}\big)_{i=1}^{2^k}$ for $k=0,1,\dots,N-1$ by
    \begin{align}
        V_i^{k,N}&=\frac{1}{\sqrt{2}}\big(V_{2i-1}^{k+1,N}+V_{2i}^{k+1,N}\big) \; \; \text{and}\label{eq:chen_V}\\
        Z_i^{k,N}&= \frac{1}{2}\big(V_{2i-1}^{k+1,N}+V_{2i}^{k+1,N}\big)+\frac{1}{4}\big(V_{2i-1}^{k+1,N}\otimes V_{2i}^{k+1,N}-V_{2i}^{k+1}\otimes V_{2i-1}^{k+1,N}\big).\label{eq:chen_Z}
    \end{align}
    
    Using the assumptions on $\mu$, it follows that $\big(V_i^{k,N},Z_i^{k,N}\big)_{i=1}^{2^k}\stackrel{\text{i.i.d.}}{\sim}\mu$ for every $k$. We now define
    \begin{equation}
        X_k^N = \bigg(\frac{1}{2^{k/2}}\sum_{i=1}^{2^k} V_i^{k, N}, \frac{1}{2^n}\sum_{i=1}^{2^k}Z_i^{k,N}\bigg)\; \; \text{for}\; \; k=0,\dots,N,
    \end{equation}
    noting that $X_0^N=\big(V_1^{0,N},Z_1^{0,N}\big)\sim \mu$. For every $N$ we have the telescoping sum
    \begin{equation*}
        X_0^N=\sum_{k=1}^{N-1}\big(X_{k-1}^N-X_k^N\big)+X_N^N,
    \end{equation*}
    and since by \cref{eq:chen_V,eq:chen_Z}
    \begin{equation*}
        X_{k-1}^N-X_k^N=\frac{1}{2^{k+1}}\bigg(0,\sum_{i=1}^{2^{k-1}}V_{2i-1}^{k,N}\otimes V_{2i}^{k,N} - V_{2i}^{k,N}\otimes V_{2i-1}^{k,N}\bigg),
    \end{equation*}
    this relation may be rewritten as
    \begin{equation*}
    \begin{split}
        X_0^N=& \Bigg(\frac{1}{2^{N/2}}\sum_{i=1}^{2^N}V_i^{N,N} \; , \;\; \sum_{k=1}^N\frac{1}{2^{k+1}}\sum_{i=1}^{2^{k-1}}V_{2i-1}^{k,N}\otimes V_{2i}^{k,N} - V_{2i}^{k,N}\otimes V_{2i-1}^{k,N}\Bigg) \\
        & +\Bigg(0 \: , \;\; \frac{1}{2^N}\sum_{i=1}^{2^N}Z_i^{N,N}\Bigg).
    \end{split}
    \end{equation*}
    
    The left-hand side has distribution $\mu$ independent of $N$, while on the right-hand side the second term tends to zero in probability, while the first term can be recognised as having the distribution
    \begin{equation*}
        \big(W_{0,1}^{D_N}, \text{Area}\big(W_{0,1}^{D_N}\big)\big)\stackrel{\text{d}}{\to} \big(W_{0,1}, A_{0,1}\big)\; \; \text{as}\; \; N\to\infty,
    \end{equation*}
    where  $W^{D_N}$ is the piecewise linear interpolation of a rescaled random walk along the $N$\textsuperscript{th} diadic partition $D_N$ of $[0,1]$. Here $\text{Area}\big(W_{0,1}^{D_N}\big)$ denotes the L\'evy area of the piecewise linear interpolation. The convergence in distribution follows from Donsker's theorem for enhanced Brownian motion \cite[Theorems 2 and 3]{Donsker_EBM}. The proof is concluded by an application of Slutsky's theorem.
\end{proof}
\begin{remark}
    To see that
    \begin{equation*}
        \sum_{k=1}^N\frac{1}{2^{k+1}}\sum_{i=1}^{2^{k-1}}V_{2i-1}^{k,N}\otimes V_{2i}^{k,N} - V_{2i}^{k,N}\otimes V_{2i-1}^{k,N}
    \end{equation*}
    has the distribution of the L\'evy area of the piecewise linear approximation of Brownian motion on the diadic partition $D^N$, one can note that for each $k=1,\dots,N$, the summand is the signed area between the piecewise linear approximation on $D^k$ and the piecewise linear approximation on the coarser partition $D^{k-1}$. For a piecewise linear path, the L\'evy area is given by the sum of these enclosed areas. Alternatively, one can show via induction that the sum may be rearranged to give the trapezium rule applied on $D^N$ to the L\'evy area of the piecewise linear approximation on the same partition. The trapezium rule in this case will be exact.
\end{remark}

If one additionally assumes that the distribution $\mu$ has finite variance and that the first marginal is Gaussian, then an alternative proof is possible utilising Wasserstein distances.
\begin{theorem}[Distributional Uniqueness of L\'evy area under Chen's relation (finite variance)]
\label{thm:DistrChenUniquenessFiniteVar}
    Suppose a distribution $\mu$ on $\reals^{d} \times \reals^{d\times d}$ has the following properties:\smallbreak
    \begin{enumerate}[label=(\roman*)]
    \item If $(V, Z) \sim \mu$, where $V \in \reals^d, \; Z \in \reals^{d \times d} ,$ then \begin{itemize}
    	\item $V \sim \normal^d(0,1),\; $
    	\item $\E[Z] = 0,$
    	\item $\var(Z) < \infty$;
    \end{itemize}
    \item If $(V_1, Z_1), (V_2, Z_2) \sim \mu$, are i.i.d. tuples, and we define
    \begin{equation}
    \label{eqn:ChenZ}
    V_3 \coloneqq \tfrac{1}{\sqrt{2}} \left( V_1 + V_2 \right), \;\;Z_3 \coloneqq \tfrac{1}{2} Z_1 + \tfrac{1}{2} Z_2 + \tfrac{1}{4} \left( V_1 \otimes V_2 - V_2 \otimes V_1 \right) ,
    \end{equation}
    then $(V_3, Z_3) \sim \mu$.\smallbreak
    \end{enumerate}
    
    Then $\mu$ is the distribution of $(W_{0,1}, A_{1})$ where $A$ is the L\'evy area process associated with a $d$-dimensional Brownian motion $W$.
\end{theorem}

\begin{proof}
Denote the distribution of $(W_{0,1}, A_{1})$ by $\nu$. We will try to estimate the 2-Wasserstein metric $\WW_2(\mu,\nu)$, and show that it is 0. Since $\WW_2$ is a metric, that means that $\mu = \nu$. Since the Wasserstein metric is defined using an infimum over all couplings $\gamma$, an upper bound on it can be obtained by picking a particular coupling and computing the $L^2$ distance between $(Z,V)$ and $(A_{0,1}, W_1)$ under that coupling.\smallbreak

Let $K \in \naturals$ be a number and let $(Z^K_n, V^K_n)$ (for $n = 1,2,\ldots,2^K$) be independent random variables drawn from $\mu$ (the superscript $K$ is an index, not a power). We will repeatedly apply \eqref{eqn:ChenZ} in a binary-tree fashion to combine all of these random variables into $(Z^0_1,V^0_1).$ Provided we are at layer $k \geq K$, consisting of $\{ (Z^k_n, V^k_n) : 1 \leq n \leq 2^k \}$, we can produce layer $k-1$ by setting
\begin{align*}
Z^{k-1}_n &\coloneqq \tfrac{1}{2} Z^k_{2n-1} + \tfrac{1}{4} Z^k_{2n} + \tfrac{1}{2} \left( V^k_{2n-1} \otimes V^k_{2n} - V^k_{2n} \otimes V^k_{2n-1} \right), \\
V^{k-1}_n &\coloneqq \tfrac{1}{\sqrt{2}} \left( V^k_{2n-1} + V^k_{2n} \right).
\end{align*}

If the random variables $\{ (Z^k_n, V^k_n) : 1 \leq n \leq 2^k \}$ are independent and $\mu$-distributed, then by (ii), so are $\{ (Z^{k-1}_n, V^{k-1}_n) : 1 \leq n \leq 2^{k-1} \}$. By induction we conclude that $(Z^0_1,V^0_1) \sim \mu$.\smallbreak

Furthermore, $Z^0_1$ can be decomposed into a sum
\[
Z^0_1 = D + \sum_{n=1}^{2^K} 2^{-K} Z^K_n,
\]
where $D$ is a rather complicated sum of correction terms of the form $ V_{2n-1} \otimes V_{2n} - V_{2n} \otimes V_{2n-1},$ but does not depend on any of the $Z^K_n$ (that is not to say it is independent of them, just doesn't contain them).\smallbreak

Since Chen's relation holds for $(A_{0,1}, W_1)$, we can perform the same procedure with independent random variables $(A_{0,1}^K(n), W_1^K(n))$ (we omit the subscripts $_{0,1}$ and $_1$ when there is no ambiguity) and obtain
\[
A^0 = D' + \sum_{n=1}^{2^K} 2^{-K} A^K(n).  
\]
where $D'$ is again a weighted sum of the $W^K(n)$.\smallbreak

Now introduce the coupling $\gamma$:\smallbreak
\begin{enumerate}[label=\arabic*)]
\item $W^K(n) = V^K_n$ for all $n = 1,\ldots,2^K$ (this is possible, since by definition $W^K(n)$ and $V^K_n$ have the same marginals).\smallbreak
\item For any fixed $n$, the dependence between $A^K(n)$ and $Z^K_n$ is unspecified, except that they both depend on $V^K_n$ (but so far it is unknown whether the marginals of $A^K(n)$ and $Z^K_n$ are equal).\smallbreak
\item The tuples $\left\lbrace \left( A^K(n), Z^K_n, V^K_n \right) : 1 \leq n \leq 2^K \right\rbrace $ are independent.\smallbreak
\end{enumerate}

Using this we can estimate the 2-Wasserstein distance
\begin{align*}
 \left( \WW_2(\mu,\nu) \right)^2 &\leq \E_{\gamma}\bigg[\ltwonorm{(A^0,W^0) - (Z^0,V^0)}^2\bigg] = \E_{\gamma}\Big[\|A^0 - Z^0\|_2^2\Big]\\
&= \E_{\gamma}\left[\Bigg\|D' + \sum_{n=1}^{2^K} 2^{-K} A^K(n) \;\; - \;\; D + \sum_{n=1}^{2^K} 2^{-K} Z^K_n\Bigg\|_2^2\right]\\
&= \E_{\gamma}\left[\Bigg\|\sum_{n=1}^{2^K} 2^{-K} \big( A^K(n) \, - \, Z^K_n \big)\Bigg\|_2^2\right] \\
&= \sum_{n=1}^{2^K} \E_{\gamma}\Big[\big\|2^{-K} \big( A^K(n) \, - \, Z^K_n \big)\big\|_2^2\Big] \\
&= \sum_{n=1}^{2^K} 2^{-2K} \E_{\gamma}\Big[\big\| A^K(n) \, - \, Z^K_n\big\|_2^2\Big] \\
&\leq 2^{-K} \cdot 4 \max \{ \var(A^K(n)), \, \var(Z^K_n) \} = C \, 2^{-K},
\end{align*}
where $C$ is some finite constant independent of $K$.\smallbreak

Since $\WW_2$ is a metric, this implies that $\mu = \nu$ as required.
\end{proof}
We conclude this section with a result that heuristically says that the error between an estimator of L\'evy area and the true distribution is bounded above by an explicit constant multiplied by the error between the estimator and one iteration of Chen-combine applied to the estimator. Let $W$ be a $d$-dimensional Brownian motion and $A \in \R^{d\times d}$ be its associated L\'evy area process.
Let $\mu$ be a measure on $\R^d \times \R^{d \times d}$ and write $(X,Z) \sim \mu$ if for all $z \in \R^{d\times d}, x \in \R^d \;$ $\P(Z \in dz, X \in dx) = \mu(dz \times dx).$ 


\begin{theorem}[Chen error bound]
\label{thm:chen_err_bd}
    Let $W$ be a $d$-dimensional Brownian motion and $A \in \R^{d\times d}$ be its associated L\'evy area process. Given a suitable square integrable, zero mean measure $\mu$ on $\R^d\times \R^{d \times d} ,$ define the ``L\'evy-error'' of $\mu$ as
    \begin{equation}
        \elevy^2 \coloneqq \inf_{Z \in \Gamma} \E\Big[ \ltwonorm{Z - A_{0,1}}^2\Big],
    \end{equation}
    where $\Gamma$ is the set of all random variables $Z$ such that $(W_{0,1}, Z) \sim \mu,$ and $\ltwonorm{ \cdot }$ is the usual $l^2$ norm defined on matrices. Let $Z_1, \, Z_2$ be independent random variables with distributions given by $(\sqrt{2}W_{0,\frac{1}{2}}, Z_1) \sim \mu \,$ and $(\sqrt{2}W_{\frac{1}{2},1}, Z_2) \sim \mu \,$. The assumed independence is possible by the independence of $W_{0,\frac{1}{2}}$ and $W_{\frac{1}{2},1}$. Write
    \[D \coloneqq \big( W_{0,\frac{1}{2}} \otimes W_{\frac{1}{2}, 1} - W_{\frac{1}{2}, 1} \otimes W_{0,\frac{1}{2}} \big), \; \text{ and } \; \hat{Z} \coloneqq \tfrac{1}{2} \big( Z_1 + Z_2 + D \big)\]
    for the Chen-combine of $Z_1$ and $Z_2.$ Using this, define the ``Chen-error'' as
    \begin{equation}
        \echen^2 \coloneqq  \inf_{Z_3 \in \Gamma} \E\Big[\ltwonormm{Z_3 - \hat{Z} }^2\Big],
    \end{equation}
    where $\Gamma$ is the set of all random variables $Z_3$  such that $(Z_3, W_1) \sim \mu$. Then 
    \begin{equation}
        \elevy \leq \big( 2 + \sqrt{2}\m\big) \echen.
    \end{equation}
\end{theorem}
\begin{proof}

    Fix $\varepsilon>0$, and since $W_1 \eqd \sqrt{2}W_{0,\frac{1}{2}} \eqd \sqrt{2}W_{\frac{1}{2},1}$ we can find random independent variables $Z_1, Z_2$ (possibly coupled to other random variables) with 
    \[
    \big(\sqrt{2}W_{0,\frac{1}{2}}, Z_1\big) \sim \mu \; \text{ and } \; \big(\sqrt{2}W_{\frac{1}{2},1}, Z_2\big) \sim \mu
    \]
    that satisfy
    \begin{equation}\label{eq:e_elevy}
        \Ex\Big[\ltwonormm{Z_1-2A_{0,\frac{1}{2}}}^2\Big]\leq \elevy^2+\varepsilon\quad\text{and}\quad \Ex\Big[\ltwonormm{Z_2-2A_{\frac{1}{2},1}}^2\Big]\leq \elevy^2+\varepsilon.
    \end{equation}

    Similarly, let $Z_3$ be a random variable for which $(W_{0,1},Z_3)\sim \mu$ and
    \begin{equation}
        \Ex\Big[\ltwonormm{Z_3-\hat{Z}}^2\Big]\leq \echen^2+\varepsilon,
    \end{equation}

    where $\hat{Z}$ is defined as in the statement of the theorem and $Z_1$ and $Z_2$ being exactly those random variables satisfying the inequalities in \cref{eq:e_elevy}. Finally, recall by Chen's relation that $A_{0,1} = A_{0,\frac{1}{2}}+A_{\frac{1}{2},1}+\frac{1}{2}D$. With the joint random variable $\big(Z_1,Z_2,Z_3,A_{0,\frac{1}{2}},A_{\frac{1}{2},1},A_{0,1},W_{0,\frac{1}{2}},W_{\frac{1}{2},1},W_{0,1}\big)$ now fully specified, we see that
    \begin{align*}
        \elevy^2 &\leq \E\Big[\ltwonormm{A_{0,1} - Z_3}^2\Big] \\
        &= \E\Big[\ltwonormm{A_{0,1} - \hat{Z} - Z_3 + \hat{Z}}^2\Big]\\
        &\leq \E\Big[\ltwonormm{A_{0,1} - \hat{Z}}^2\Big] + \E\Big[\ltwonormm{Z_3 - \hat{Z}}^2\Big] -2 \E\Big[\big\langle A_{0,1} - \hat{Z} \, , \, Z_3 - \hat{Z} \big\rangle\Big].
    \end{align*}
    The second term may be bounded as $\E{ \ltwonormm{Z_3 - \hat{Z}}^2} \leq \echen^2 + \varepsilon$. By construction 
    \begin{align*}
        A_{0,1} - \hat{Z} &= A_{0,\frac{1}{2}} + A_{\frac{1}{2},1} + \frac{1}{2}D - \hat{Z}\\
        &= \Big(A_{0,\frac{1}{2}} - \frac{1}{2}Z_1\Big) + \Big(A_{\frac{1}{2},1} - \frac{1}{2}Z_2\Big),
    \end{align*}
    allowing us to write
    \begin{align*}
        \E\Big[\ltwonormm{A_{0,1} - \hat{Z}}^2\Big] & = \Ex\Bigg[\bigg\|\Big(A_{0,\frac{1}{2}} - \frac{1}{2}Z_1\Big) + \Big(A_{\frac{1}{2},1} - \frac{1}{2}Z_2\Big)\bigg\|_2^2\Bigg] \\[2pt]
        & = \frac{1}{4} \Bigg( \Ex\Big[\ltwonormm{Z_1 - 2 A_{0, \frac{1}{2}}}^2\Big]+ \Ex\Big[\ltwonormm{Z_2 - 2 A_{\frac{1}{2}, 1}}^2\Big]\\[-2pt]
        &\hspace{12.5mm} + \E{\big\langle Z_1 - 2 A_{0, \frac{1}{2}} \, , \, Z_2 - 2 A_{\frac{1}{2}, 1} \big\rangle} \Bigg) \\
        &\leq \frac{1}{2} \elevy^2 + \frac{1}{2} \varepsilon.
    \end{align*}
    The cross terms vanish by independence and the mean zero property of all random variables involved. Finally we can bound the third term using the Cauchy-Schwarz inequality
    \begin{align*}
        -2 \E\Big[\big\langle A_{0,1} - \hat{Z} \; , \; Z_3 - \hat{Z} \big\rangle\Big] &\leq 2 \Ex\Big[\ltwonormm{A_{0,1} - \hat{Z}}^2\Big]^{\frac{1}{2}} \Ex\Big[\ltwonormm{Z_3 - \hat{Z}}^2\Big]^{\frac{1}{2}} \\
        &\leq \sqrt{2} \big(  \elevy^2 +  \varepsilon \big)^{\frac{1}{2}} \big( \echen^2 + \varepsilon \big)^{\frac{1}{2}} 
        %
    \end{align*}
    By substituting all bounds into the original inequality, and taking $\varepsilon\to 0$, we obtain
    \begin{equation*}
        \elevy^2 \leq \tfrac{1}{2} \elevy^2 + \sqrt{2} \elevy \echen + \echen^2 = \Big( \tfrac{1}{\sqrt{2}} \elevy + \echen \Big)^2.
    \end{equation*}
    Since all quantities are non-negative, we see that $\left( 1 - \frac{1}{\sqrt{2}} \right) \elevy \leq \echen$ as required.
    \end{proof}
    \begin{lemma}\label{lem:BridgeLAFlip_sm}
    Let $W$ be a $d$-dimensional Brownian motion on $[0,1]$ and let $B, H, b$ be the corresponding derived processes from \cref{def:BBridge}. Fix some $\xi \in \{-1, 1 \}^d $, and let $H^\prime$, and $b^\prime$ be the space-time and space-space L\'evy area processes associated with the process $\xi \odot B = \{\xi \odot B_t\}_{t\in[0,t]}$, where $\odot$ denotes the Hadamard (coordinate-wise) product. Then
    \[
         \{ B_t \}_{t \in [0,1]} \eqd \{ \xi \odot B_t \}_{t \in [0,1]}, \;\;  H^\prime = \xi \odot H, \;\; b^\prime = (\xi \otimes \xi) \odot b, \;\; \text{ and } \;\; (H^\prime,b^\prime)\eqd (H, b).
    \]
\end{lemma}
    \begin{proof}
    By definition of $H$ and $b$ (in vector form)
    \begin{align*}
        H'_{0,t} &= \frac{1}{t} \int_0^t \xi \odot B_{0,u} \, du = \xi \odot H_{0,t} \\[2mm]
        b'_{0,t} &= \int_0^t ( \xi \odot B_{0,u}) \otimes d( \xi \odot B_u ) \\
        &= (\xi \otimes \xi) \odot \int_0^t B_{0,u} \otimes dB_u = (\xi \otimes \xi) \odot b_{0,t}.
    \end{align*}
\end{proof}
\section{Characteristic Function and Joint Moments}\label{sec:dist_properties}
Here we briefly complete the definition of the joint characteristic function and L\'evy area found in \cref{main_theorem_into}, and also provide a proof of \cref{prop:moments}. Part of the proof of the second result comes as a corollary of the form of the characteristic function found in \cref{main_theorem_into} and \cref{lem_antisym_matrix_decomposition}.
\begin{theorem}[Characteristic function of Brownian motion and L\'evy area]
\label{main_theorem_into}
Let $W_t = (W_t^{(1)},\dots W_t^{(d)})$ be a $d$-dimensional Brownian motion and let $A^{(j_1,j_2)}_t$ be the corresponding L\'evy area of $W_t^{(j_1)}$ and $W_t^{(j_2)}$. Let $\mu \in \mathbb{R}^d$ and $\Lambda = \{\Lambda_{i,j}\}_{1\leq i< j\leq d}\in \mathbb{R}^{\frac{d(d-1)}{2}}$. The joint characteristic function of coupled Brownian motion $W_{t}$ and the L\'{e}vy area $A_{t}$ at time $t$
    \begin{equation*}
        \Phi_{(W, A)}(t, \mu, \Lambda) := \mathbb{E} \bigg[ \exp \bigg(i \sum_{i=1}^d \mu_i W_{0,t}^{(i)}+ i \sum_{1\leq j_1< j_2\leq d} \Lambda_{j_1,j_2} A^{(j_1,j_2)}_{t} \bigg)\bigg]
    \end{equation*}
    admits the following formula
    \begin{align*}
        \Phi_{(W, A)}(t, \mu, \Lambda) = & \bigg( \prod_{i=1}^{d_1} \frac{1}{\cosh(\frac{\eta_i}{2}t)} \bigg)  \exp \bigg( \bigg[ \sum_{i=1}^{d_1} - \frac{1}{\eta_i} ((R\mu)_{2i-1}^2
        \\
        &\hspace{45mm}+(R\mu)_{2i}^2)\tanh\big(\frac{\eta_i}{2}t\big) \bigg] - \frac{1}{2}t \sum_{i=1}^{d_0} (R\mu)^2_{2d_1+i}\bigg)
    \end{align*}
    where $R,\ \eta,\ d_0,\ d_1$ are defined in \cref{lem_antisym_matrix_decomposition}.
\end{theorem}

\begin{proof}
The result was first proved in \cite{Helmes1983levy} using a probabilistic approach. Recently, \cite{Jiajie2023levy} proposed a rough paths approach. 
\end{proof}
\begin{lemma}[Decomposition of anti-symmetric matrix]\label{lem_antisym_matrix_decomposition}
For any $d \times d$ anti-symmetric real-valued matrix $\Lambda$, let $(\pm \eta_{1}i,\dots, \pm \eta_{d_1}i)$ be the set of non-zero conjugate eigenvalue pairs of $\Lambda$ with $\eta_i > 0$ and $\eta_1\geq,\dots, \geq \eta_{d_1}$. Let $d_0 = d-2d_1$ be the algebraic multiplicity of the eigenvalue 0 of $\Lambda$ (if $\Lambda$ does not have zero eigenvalues, then $d_0=0$). Then there exists an orthogonal matrix $R$, such that the following decomposition of $\Lambda$ holds:
\begin{eqnarray*}
\Lambda = R^T \Sigma R,
\end{eqnarray*}
where $\Sigma$ is in the form that
\begin{align*}
\Sigma &= \begin{pmatrix}
\Sigma_0,& \mathbf{0}_{d-d_0, d_0}\\
\mathbf{0}_{d_0, d-d_0}&\mathbf{0}_{d_0,d_0}
\end{pmatrix} \\[3mm]
& \coloneqq \begin{pmatrix}
0 & -\eta_{1} & 0 &0& \dots & 0 &0 & 0 & \dots&0\\
\eta_{1} & 0 & 0 & 0 & \dots &  0&0& 0 & \dots&0 \\
0 & 0  & 0& -\eta_{2} & \dots & 0&0& 0 & \dots&0\\
0 &0& \eta_{2} &0 &\dots &0&0&0 & \dots&0\\
\vdots & \vdots & \vdots & \vdots&\ddots&\vdots &\vdots&\vdots &\dots& \vdots\\
0 & 0 & \dots & 0&\dots &0& -\eta_{d_1} &0& \dots  & 0  \\
0 & 0 & \dots & 0&\dots &\eta_{d_1} & 0 &  0& \dots  & 0 \\
0 & 0 & \dots & 0 & 0 &0 & 0 &0& \dots  & 0 \\
0 & 0 & \dots & 0 & 0 &0 & 0 & \vdots & \ddots&\vdots\\
0 & 0 & \dots & 0 & 0 &0 & 0 & 0 & \dots&0\\
\end{pmatrix},
\end{align*}
where $\Sigma_0$ is the block diagonal matrix with all non-zero $\eta_i$.
\end{lemma}

\begin{proof}
   See \cite{Greub1967}.
\end{proof}

In the following proposition, we denote by $\Pbf$ the output distribution of the BF generator given inputs $w$ and $\theta$, and the joint distribution of Brownian increments and fake area defined by the generator is denoted by $\Pbfj$.
\begin{proposition}\label{prop:moments}
The following facts about the joint and conditional distributions of Brownian motion, L\'evy area, and the BF generator hold.\smallbreak
\begin{enumerate}[label=\arabic*)]
\item If $r\sim\mathrm{Rad}\big(\frac{1}{2}\big)$ then $(rW_{0,t},rA_{0,t})\eqd(W_{0,t},rA_{0,t})\eqd(W_{0,t},A_{0,t})$.\smallbreak
\item  For any $n_i,n_{ij},m_{ij}\in\N,\ 1\leq i<j\leq d$ it holds that
{\allowdisplaybreaks
\begin{align}
    \Ex_{(W_{0,1},A_{0,1})\sim\P_{(W_{0,1},A_{0,1})}}\left[\prod_{i=1}^d\big(W_{0,1}^{(i)}\big)^{n_i}\bigg(\prod_{1\leq i<j\leq d}\big(A_{0,1}^{(i,j)}\big)^{n_{ij}}\bigg)\right]&=0,\label{eq:true_moments}\\
    \Ex_{(W_{0,1},\tA)\sim\Pbfj}\left[\prod_{i=1}^d\big(W^{(i)}\big)^{n_i}\bigg(\prod_{1\leq i<j\leq d}\big(\tA^{(i,j)}\big)^{n_{ij}}\bigg)\right]&=0\label{eq:fake_moments},\\
    \Ex_{A_{0,1}\sim\P_{A_{0,1}\mid W_{0,1}=w}}\left[\prod_{1\leq i<j\leq d}\big(A_{0,1}^{(i,j)}\big)^{m_{ij}}\right]&=0,\label{eq:true_cond_moments}\\
    \Ex_{\tA\sim\Pbf}\left[\prod_{1\leq i<j\leq d}\big(\tA^{(i,j)}\big)^{m_{ij}}\right]&=0\label{eq:fake_cond_moments},
\end{align}}
provided that $\sum_{i=1}^d n_i+\bigg(\sum_{1\leq i<j\leq d}n_{ij}\bigg)\; \text{and}\; \sum_{1\leq i<j\leq d}m_{ij}$ are odd.
\end{enumerate}
\end{proposition}


\begin{proof}
    To see item $(1)$, fix $\mu \in \mathbb{R}^d$ and $\lambda = \{\lambda_{i,j}\}_{1\leq i< j\leq d}\in \mathbb{R}^{\frac{d(d-1)}{2}}$. We claim that
    \begin{equation}\label{eq:char_negatives}
        \Phi_{(W, \vA)}(t, \mu, \lambda)=\Phi_{(W, \vA)}(t, -\mu, -\lambda)=\Phi_{(W, \vA)}(t, \mu, -\lambda).
    \end{equation}
    
    Let $\Lambda$ stand for the $d\times d$ anti-symmetric matrix with the above diagonal elements given by $\lambda$. Let $\Lambda$ have a decomposition $\Lambda = R^T\Sigma R$ as in \cref{lem_antisym_matrix_decomposition}. Then the corresponding decomposition for $-\Lambda$ is given by $\tilde{R}^T\Sigma \tilde{R}$, where for $1\leq i \leq d_1$
    \begin{equation*}
        \tilde{R}_{2i,\cdot}\coloneqq R_{2i-1,\cdot}\quad\text{and}\quad\tilde{R}_{2i-1,\cdot}\coloneqq R_{2i,\cdot}\,.
    \end{equation*}
    
    It is clear from the form of the characteristic function, that pairwise swapping the first $2d_1$ rows of $R$ leaves the final value unchanged. Since all the terms involving $\mu$ are squared, the characteristic function is also invariant under taking the negative of $\mu$. Item $(1)$ now follows from \cref{eq:char_negatives}. For item $(2)$, \cref{eq:true_moments} follows immediately from item $(1)$, and \cref{eq:true_cond_moments} follows from the fact that the conditional characteristic function of L\'evy area given $W_{0,t}=\vw$ is purely real \cite{Wiktorsson}. \cref{eq:fake_cond_moments} follows by independence of $\xi_0$ and the observation that its power will be odd. We finally turn our attention to \cref{eq:fake_moments}, where for simplicity we take $t=1$. We recall that $\tA_{0,1}\sim\Pbfj$ is defined by $\BFf{W_{0,1},H_{0,1},\tb_{0,1},\xi_,\vxi}$, where the only dependence is between $H_{0,1}$ and $\tb_{0,1}$. Indeed, by expanding each $\tA^{(i)}$ we see that it is enough to show that
    \begin{equation}\label{eq:moment_expansion}
        \Ex\left[\prod_{i=1}^d\big(W^{(i)}\big)^{n_i}\bigg(\prod_{1\leq i<j\leq d}(\xi_0)^{n_{ij}}\big(\xi_i H^{(i)}W^{(j)}\big)^{k_{ij}}\big(\xi_j H^{(j)}W^{(i)}\big)^{l_{ij}}\big(\xi_i\xi_j \tb^{(i,j)}\big)^{p_{ij}}\bigg)\right]=0,
    \end{equation}
    where $k_{ij}+l_{ij}+p_{ij}=n_{ij}$. We consider three exhaustive, but overlapping cases.\smallbreak
    \begin{enumerate}[label=\roman*)]
        \item $\sum_{1\leq i<j\leq d}n_{ij}$ is odd.\smallbreak
        \item For some $1\leq i \leq d $, the power of $W^{(i)}$ or $\xi_{i}$ is odd.\smallbreak
        \item For every $1\leq i \leq d$, the power of $W^{(i)}$ and $\xi_{i}$ is even.\smallbreak
    \end{enumerate}
    If i) holds, then the power of $\xi_0$ is odd and by independence \cref{eq:moment_expansion} is true. Item ii) is similar. To conclude, we will show that if iii) holds, then i) must also hold. We observe that the power of $W^{(i)}$ is given by
    \begin{equation*}
        n_i+\sum_{1\leq j <i}k_{ji} + \sum_{i<j\leq d}l_{ij},
    \end{equation*}
    and the power of $\xi_i$ is given by
    \begin{equation*}
        \sum_{1\leq j <i}l_{ji} +p_{ji}+ \sum_{i<j\leq d}k_{ij} + p_{ij}.
    \end{equation*}
    
    Under the assumption that all these powers are even, we may add them together to see that
    \begin{equation*}
        \sum_{i=1}^d n_i + \bigg(\sum_{1\leq j <i}k_{ji}+l_{ji}+p_{ji}+\sum_{i<j\leq d}k_{ij}+l_{ij}+p_{ij}\bigg)= \sum_{i=1}^d n_i+2\bigg(\sum_{1\leq i<j\leq d}n_{ij}\bigg)
    \end{equation*}
    is even. As the sum over all $n_i$ and $n_{ij}$ is odd by assumption, it must be the case that
    \begin{equation*}
        \sum_{1\leq i<j\leq d}n_{ij}
    \end{equation*}
    is odd.
\end{proof}


\section{Test metrics}
\label{appendix:test metrics}
We list in this section the test metrics we used to assess the performance of our generative model. Throughout the section let $X$, $Y$ be $d$-dimensional random variables on the metric space $(M, m)$ and let $\mu \coloneqq \P_X$, $\nu \coloneqq \P_Y$ be the induced probability measure. 
We denote by $\mu_i$ and $\nu_i$ the measure induced by the marginal distributions of $X^{(i)}$ and $Y^{(i)}$ respectively.\smallbreak
\begin{enumerate}[label=\arabic*)]
    \item Marginal 2-Wasserstein metric. For $1\leq i\leq d$, The 2-Wasserstein distance between $X^{(i)}$ and $Y^{(i)}$ is given by
    \begin{equation*}
        \WW_2(\mu, \nu) = \bigg ( \inf_{\gamma \in \Gamma(\mu, \nu)} \int m(x,y)^2 \m d\gamma(x,y)  \bigg )^{\frac{1}{2}}
    \end{equation*}
    where $\Gamma(\mu, \nu)$ is the set of all joint measures on $M \times M$ such that the marginal measure corresponds to $\mu$ and $\nu$, i.e.
    \begin{align*}
        \int_M \gamma(x,y)\m dy & = \mu(x)
        \\
        \int_M \gamma(x,y)\m dx & = \nu(y)
    \end{align*}
    Wasserstein metric is a way of assessing the difference between two distributions, however, in practice the estimation is often intractable when $d$ is high as it needs to compute $m(x,y)$ for all samples $x$ and $y$. To accommodate this issue, we compute the $\WW_2$ distance between the marginal distribution $X^{(i)}$ and $Y^{(i)}$ only.\smallbreak
    \item Cross moment metric. We compare the difference of the fourth moments between the real and generated L\'{e}vy area. Let $\{\vX_i\}_{i=1}^N$ be samples of the random variable $X$, for any $(i_1, i_2, i_3, i_4) \in \{1,\dots, d\}^4$ we estimate the cross moment $\E [X^{(i_1)}X^{(i_2)}X^{(i_3)}X^{(i_4)}]$ by
    \begin{equation*}
        \frac{1}{N^4} \sum_{1\leq j_1, j_2, j_3, j_4 \leq N} \vX_{j_1}^{(i_1)}\vX_{j_2}^{(i_2)}\vX_{j_3}^{(i_3)}\vX_{j_4}^{(i_4)}
    \end{equation*}
    If $X$ happens to be the L\'{e}vy area of Brownian motions we do know the analytical form of each cross-moment. One can derive it using the basic tools of rough path theory, i.e. the expected signature of Brownian motion. This metric is used to assess the joint fitting of generated L\'{e}vy area without using estimations of real data.\smallbreak
    \item Characteristic Function Distance and Maximum Mean Discrepancy distance (MMD). In \cite{CFMMD} it is shown that the characteristic function distance $\text{CFD}_{\Lambda}$ when using an $L^2$ norm instead of an $L^1$ norm is equivalent to an MMD distance with a certain kernel. For example, if $\Lambda$ is distributed as a Gaussian then the corresponding kernel is the Gaussian kernel. For our tests we use the MMD distance with both a Gaussian and polynomial kernel, with the Fourier series expansion with a high truncation level taken to be the ground truth.\smallbreak
    \item Empirical Unitary Characteristic Function Distance. We parametrize the measure $\mathcal{M}$ on the space of linear transformations $\mathbb{R}^{d+a}$ onto $\mathfrak{g}_m$ by $M$ linear maps, where $d,a$ and $m$ denotes the Brownian dimension, L\'evy dimension, and unitary Lie degree respectively. Then, we train $\mathcal{M}$ to maximize the EUCFD distance between real and generated data. In practice, we set $m=8$, $M=128$, and train $\mathcal{M}$ for $2000$ iterations, the EUCFD computed on an independent test set is used for model assessment.\smallbreak
\end{enumerate}
\section{Training procedure and hyperparameter tuning}\label{sec:param_tuning}
We performed the training procedure as illustrated in \cref{alg:LevyGAN}.\smallbreak

On the generator side, we used a FNN with $3$ hidden layers and $16$ hidden dimensions per layer. The activation function is chosen to be LeakyRelu with slope $0.01$. The dimension of the noise vector is set to be $4$.\smallbreak

On the discriminator side, we parameterize $128$ linear maps onto the Lie algebra of degree $3$ to mimic the empirical distribution used to compute UCFD mentioned in \cref{sec: ucf}. The total number of training iterations is set to be $2500$, where we observed the convergence on the marginal $2$-Wasserstein metric on real data. We optimize both the generator and discriminator using Stochastic Gradient Descent and Adam optimizer. We set the batch size to $2^{13}$ and the learning rate for generator/discriminator is set to be $0.001$/$0.01$ respectively. Both learning rates decay for each $500$ iteration. Finally, we set $\text{iter}_d$ to be $3$.
\smallbreak

Regarding the choice of some of the parameters, we have conducted a wide range hyperparameter grid search for training the L\'evyGAN. Tuning is done on both the generator and discriminator sides. The model selection is based on the marginal $2$-Wasserstein metric. We provide here a complete grid for interested readers.
\begin{table}[H]
\begin{center}
\begin{tabular}{@{}ll@{}}
\toprule
Hyperparameter & Grid values \\
\midrule
Generator hidden layers & $2, 3, 5$ \\
Generator hidden dimension & $8, 16$ \\ 
Generator noise size & $4, 8$ \\
Generator slope of LeakyRelu & $0.01, 0.2$\\
Discriminator Lie degree & $2, 3, 5$ \\
Discriminator batch size & $16, 64, 128$ \\
 \bottomrule
\end{tabular}
\end{center}\vspace*{2.5mm}
\caption{Grid for hyperparameter tuning}
\end{table}\smallbreak
\section{Foster's Moment-Matching Method}
\label{sec:fosters_method}

\begin{definition}[Foster's approximation of L\'evy area]
\label{def:fosters_method}
For any $0 \leq s < t $ and $d \geq 2$ we define the antisymmetric tensor $\tilde{A}_{s,t}$ with entries
\[
\tilde{A}^{(i,j)}_{s,t} = H^{(i)}_{s,t} W^{(j)}_{s,t} - W^{(i)}_{s,t} H^{(j)}_{s,t} + 12 \Big( K^{(i)}_{s,t} H^{(j)}_{s,t} - H^{(i)}_{s,t} K^{(j)}_{s,t} \Big) + \tilde{a}^{(i,j)}_{s,t}, \;\; \text{ for } \; 1 \leq i,j \leq d
\]
where
\begin{itemize}
    \item $W$ is a $d$-dimensional Brownian motion and $H$ is a space-time bridge L\'evy area on $[s,t]$.\smallbreak
    \item $K_{s,t} \in \reals^d$ is the space-time-time L\'evy area of the Brownian bridge between $s$ and $t$, is distributed as $K_{s,t} \sim \normal^d \! \left( 0, \, \frac{1}{720} (t-s) \right)$, with $K$ and $(W, H)$ being independent.\smallbreak
    \item $\tilde{a}$ is an approximation of the Brownian arch L\'evy area (see \cite{FosterThesis}, definition 4.1.14) constructed as shown below.
\end{itemize}
\[
\tilde{a}^{(i,j)}_{s,t} = \begin{cases}
	\sigma^{(i,j)}_{s,t} \xi^{(i,j)}_{s,t} & \text{ if } i < j \\
	-\sigma^{(i,j)}_{s,t} \xi^{(i,j)}_{s,t} & \text{ if } j < i \\
	0 & \text{ if } i = j
\end{cases}
\]

with the independent random variables $\sigma^{(i,j)}_{s,t}$ and $\xi^{(i,j)}_{s,t}$ defined for $1 \leq i < j \leq d$ according to

\[
\xi^{(i,j)}_{s,t} \sim \begin{cases}
\Uni\! \left[ -\sqrt{3}, \, \sqrt{3} \right] \; & \text{ with probability } p \\
\Rad(1/2) \; & \text{ with probability } 1-p
\end{cases} \hspace{1cm} \text{where} \;\; p \coloneqq \frac{21130}{25621}
\]
and
\[
\sigma^{(i,j)}_{s,t} = \sqrt{ \frac{3}{28} \left(C^{(i)} + c \right) \left(C^{(j)} + c \right) (t-s)^2 + \frac{1}{28} (t-s) \left( \left(12 K^{(i)}_{s,t} \right)^2 + \left(12 K^{(i)}_{s,t} \right)^2 \right) },
\]
where the $C^{(i)} \sim \Exp \! \left( \frac{15}{8} \right), \; (1 \leq i \leq d) $ are i.i.d. random variables and $c \coloneqq \frac{1}{\sqrt{3}} - \frac{8}{15}.$

\end{definition}

\begin{theorem}
Let $d \in \{ 2, 3 \}$ and let $\tilde{A}$ be as in \cref{def:fosters_method}. Then $\tilde{A}$ matches all the fifth and lower moments of the L\'evy area $A$ conditional on $W$ and $H$. That is, for any $n_1, n_2, n_3 \geq 0$ with $n_1 + n_2 + n_3 \leq 5$ 
{\small
\begin{align*}
    &\E{ \left( \tilde{A}^{(1,2)}_{s,t} \right)^{n_1} \left( \tilde{A}^{(2,3)}_{s,t} \right)^{n_2} \left( \tilde{A}^{(3,1)}_{s,t} \right)^{n_3} \bigm\vert W_{s,t}, H_{s,t} } = \\
    &\E{ \left( A^{(1,2)}_{s,t} \right)^{n_1} \left( A^{(2,3)}_{s,t} \right)^{n_2} \left( A^{(3,1)}_{s,t} \right)^{n_3} \bigm\vert W_{s,t}, H_{s,t} }.
\end{align*}
}
\end{theorem}
\begin{proof}
    See \cite{FosterThesis}, page 190.
\end{proof}
\section{SDE Numerical Schemes}
This section details the various numerical schemes with which we perform our tests. We start by recalling the definition of weak and strong errors and discuss why we are only able to analyse the former.
\subsection{Weak vs Strong Error}\label{sec:weak_vs_strong}
Strong error measures the discrepancy between the true sample paths of the process and the approximate sample paths, while weak error measures the similarity between the distributions of the true and approximate solutions. Since our simulation methods are not exact, nor do they correspond to some approximation of a true sample of L\'evy area (in comparison to a truncated Fourier series expansion for example), we cannot measure the strong error for schemes involving fake L\'evy area. We can, however, measure the weak error of the discretisation schemes. Concretely, a discretisation scheme $\{\hat{X}_{k}\}_{0\leq k \leq N}$ is said to converge weakly with order $\alpha$ if for any polynomial $p$
\begin{equation}
    \abs{\Ex\big[ p(\hat{X}_N)-p(X_T)\big]} \leq C_p h^\alpha,
\end{equation}
for some constant $C_p>0$ and for all sufficiently small step sizes $h\coloneqq \tfrac{T}{N}$. Convergence with weak order $\alpha$ will be denoted by $O(h^\alpha)$. The following subsections introduce three discretisation schemes that will be used for our numerical results.
\subsection{Discretisation Schemes}\label{sec:discretisation_schemes}
Here we recall the definitions of the numerical schemes used for the results presented in \cref{sec:sde_results}.
\subsubsection{Milstein's Method}\label{sec:milstein}
Milstein's method is derived from the second order (It\^o) Taylor expansion of the SDE in \cref{eq:SDE_1}. It is defined as follows: fix $N\geq 1$, set $\hat{X}_0=x_0$, and for all $0\leq k \leq N-1$ construct $\hat{X}_{k+1}$ recursively via the relation
\begin{equation}
\begin{split}
    \hat{X}_{k+1}=&\hat{X}_k + f(\hat{X}_k)h+\sum_{i=1}^dg_i(\hat{X}_k)\Delta W_k^{(i)}\\
    &+\sum_{i,j=1}^dg_i^\prime(\hat{X}_k) g_j(\hat{X}_k)\left(\frac{1}{2}(\Delta W_k^{(i)})(\Delta W_k^{(j)})+A_k^{(i,j)}-\frac{1}{2}\delta_{ij}h\right),
\end{split}
\end{equation}
where $h\coloneqq\frac{T}{N}$, $\Delta W_k^{(i)}\coloneqq W_{t_{k+1}}^{(i)}-W_{t_k}^{(i)}$, $\delta_{ij}$ is the Kronecker delta, and
\begin{equation}\label{eq: mil LA}
    A_k^{(i,j)} \coloneqq \frac{1}{2}\left(\int_{t_k}^{t_{k+1}}(W_s^{(i)}-W_{t_k}^{(i)})\m dW_s^{(j)}-\int_{t_k}^{t_{k+1}}(W_s^{(j)}-W_{t_k}^{(j)})\m dW_s^{(i)}\right)
\end{equation}
is the L\'evy area. Milstein's method, under certain conditions, is known to converge with both weak and strong order $O(h)$. Meanwhile, removing the L\'evy area terms leaves the weak order unchanged, but reduces the strong order to $O(\sqrt{h})$. Since we are concerned exclusively with weak convergence, we will consider Milstein's method without L\'evy area.
\subsubsection{Strang log-ODE}
The Strang log-ODE method, introduced in \cite{foster2024high}, is a higher order method that incorporates a term involving fake L\'evy area. To define this scheme we must first convert the SDE in \cref{eq:SDE_1} into Stratonovich form, that is
\begin{equation}\label{eq:base SDE strat}
    dX_t = \bar{f}(X_t\m dt + \sum_{i=1}^dg_i(X_t)\circ dW_t^{(i)},\ \ X_0=x_0,
\end{equation}
where the Stratonovich drift $\bar{f}$ is given by
\begin{equation*}
    \bar{f}(x)\coloneqq f(x)-\sum_{i=1}^d g_i'(x)g_i(x).
\end{equation*}
As in the Milstein scheme we fix $\hat{X}_0=x_0$, and recursively define
\begin{equation}
    \hat{X}_{k+1}=\exp\left(\tfrac{1}{2}h\bar{f}\right)\exp\left(\sum_{i=1}^dg_i\Delta W_k^{(i)}+\sum_{i<j}(g_j^\prime g_i-g_jg_i^\prime)A_k^{(i,j)}\right)\exp\left(\tfrac{1}{2}h\bar{f}\right)\left(\hat{X}_k\right),
\end{equation}
where $\exp\left(Cf\right)(x)$ denotes the time $1$ solution to the following ordinary differential equation
\begin{equation*}
    \frac{dy}{dt}=Cf(y),\ \ y_0=x.
\end{equation*}
For the L\'evy area term we will consider three different weak approximations: a Rademacher random variable which matches the variance of L\'evy area, Foster's method, and our generative modelling method. The first method is based on the use of Rademacher random variables in Talay's scheme \cite{Talay}.

\subsection{Antithetic MLMC}
Normally, in order to achieve optimal computational complexity of the MLMC it is required that the underlying discretisation scheme has strong order $O(h)$, see \cite{GilesMLMC}, which requires strong simulation of L\'evy area when the SDE is more than one dimensional. However, there exists clever reformulation of the standard MLMC which achieves optimal computational complexity using the no area Milstein scheme, despite only having a strong convergence rate of $O(\sqrt{h})$. Antithetic MLMC, proposed in \cite{GilesSzpruch}, was originally motivated by the difficulty of simulating L\'evy area exactly. Now, as mentioned in the previous subsection, on each level the coarse and fine path are generated using the same underlying Brownian motion. It is not in fact necessary to use the same estimator on both the coarse and fine paths. It suffices that $Y_{l}^{\cdot,l}\eqd Y_{l}^{\cdot,l+1}$.
For the antithetic method, the paper \cite{GilesSzpruch} uses two different paths on the fine level: the fine path itself and an antithetic twin. The antithetic twin is obtained by pairwise swapping the increments of the fine path; essentially this means that $\hat{X}^{(f)}+\hat{X}^{(a)}\approx 2\hat{X}^{(c)}$, where $\hat{X}^{(f)}$ denotes the fine path, $\hat{X}^{(a)}$ the antithetic path, and $\hat{X}^{(c)}$ the coarse path. The resulting effect is that the variance of $\tfrac{1}{2}(\varphi(\hat{X}^{(f)})+\varphi(\hat{X}^{(a)}))-\varphi(\hat{X}^{(c)})$ should be small. For more details, we refer the reader to \cite{GilesSzpruch}.
\subsection{The log-Heston SDE}
We recall the definition of the log-Heston model as the following two-dimensional SDE
\begin{equation}\label{eq:logheston_sde}
    \begin{split}
        dU_t &= \left(r-\frac{1}{2}V_t\right)dt+\sqrt{V_t}dW_t^{(1)},\ \ U_0\in\R\\
        dV_t&= \kappa(\theta-V_t)dt+\sigma\sqrt{V_t}dW_t^{(2)},\ \ V_0>0.
    \end{split}
\end{equation}
The discounted payoff of a European call option for a price $S$ with $d\log(S):=U$ is
\begin{equation*}
    \varphi(S):=e^{-rT}\left(e^{U_T}-K\right)^+.
\end{equation*}
The price at $t=0$ of this option may be written as
\begin{equation}
    C_0\coloneqq \Ex\left[\varphi(S)\right]=S_0\Pi_0-e^{-rT}K\Pi_1,
\end{equation}
where the factors $\Pi_0$ and $\Pi_1$ area given by
\begin{equation}\label{eq:heston_factors}
\begin{split}
    \Pi_0&=\frac{1}{2}+\frac{1}{\pi}\int_0^\infty\mathfrak{R}\left[\frac{e^{i\omega K}\Psi_{U_T}(\omega - i)}{i\omega\Psi_{U_T}(-i)}\right]d\omega\\
    \Pi_1&=\frac{1}{2}+\frac{1}{\pi}\int_0^\infty\mathfrak{R}\left[\frac{e^{i\omega K}\Psi_{U_T}(\omega)}{i\omega}\right]d\omega.
\end{split}
\end{equation}
Here $\Psi_{\log Y_T}$ denotes the characteristic function of $U_T$. The characteristic function itself has an analytic form given by
\begin{equation*}
    \Psi_{U_T}(\omega)=\exp\left\{C(\omega)\theta+D(\omega)V_0+i\omega\log\left(S_0e^{rT}\right)\right\},
\end{equation*}
where the functions $C(\cdot)$ and $D(\cdot)$ are defined as
\begin{align*}
    C(\omega)&\coloneqq \kappa \left[b_1T-\frac{2}{\sigma^2}\log\left(\frac{1-b_2e^{-aT}}{1-b_2}\right)\right]\\
    D(\omega)&\coloneqq b_1\frac{1-e^{-at}}{1-b_2e^{-at}},
\end{align*}
with constants $a,b_1$ and $b_2$ given by
\begin{align*}
a&\coloneqq \sqrt{\kappa^2+\sigma^2\omega(\omega-i)}\\
b_1&\coloneqq \frac{\kappa-a}{\sigma^2}\\
b_2&\coloneqq \frac{\kappa - a}{\kappa_a}.
\end{align*}
For full details of the derivation of this formula, we refer the reader to \cite{logHestonMatlab}. While the preceding expression may appear complicated, it requires only the evaluation of the deterministic integrals in \cref{eq:heston_factors}, which are computed in practice using quadrature.\smallbreak

What remains to full specify the numerical schemes is to compute the one step recursion for the various numerical schemes applied log-Heston SDE from \cref{eq:logheston_sde}. Full details and derivations of the vector field derivatives can be found in \cite{Terada}. The no-area Milstein update is given by
\begin{equation}
    \begin{split}
        \hat{U}_{k+1}&=\hat{U}_k+(r-\tfrac{1}{2}\hat{V}_k)h+\sqrt{\hat{V}_k}\Delta W_k^{(1)}+\tfrac{1}{4}\sigma\Delta W_k^{(1)}\Delta W_k^{(2)}\\
        \hat{V}_{k+1}&=\hat{V}_k+\kappa(\theta-\hat{V}_n)h+\sigma\sqrt{\hat{V}_k}\Delta W_k^{(2)}+\frac{1}{4}\sigma^2\left((\Delta W_k^{(2)})^2-h\right).
    \end{split}
\end{equation}
The Strang log-ODE recursion is more involved as it involves solving two ordinary differential equations: one for the Stratonovich drift and one involving diffusion terms. Fortunately, these are analytically solvable, with the resulting scheme given by
\begin{equation}
    \begin{split}
        \Tilde{V}_{k+1}^1 &= (\hat{V}_k+\xi)e^{-\frac{\kappa h}{2}}+\xi\\
        \Tilde{U}_{k+1}^1 &= \hat{U}_k+\frac{1}{2\kappa}(\hat{V}_k-\xi)(e^{-\frac{\kappa h}{2}}-1)+\tfrac{h}{2}(r-\tfrac{\xi}{2})\\
        \Tilde{V}_{k+1}^2&=\left(\sqrt{\Tilde{V}_{k+1}^1}+\tfrac{\sigma}{2}\Delta W_k^{(2)}\right)^2
    \end{split}
\end{equation}
\begin{align*}
    \Tilde{U}_{k+1}^2&=\Tilde{U}_{k+1}^1 + \sqrt{\Tilde{V}_{k+1}^1}\Delta W_k^{(1)}+\tfrac{\sigma}{4}\Delta W_k^{(1)}\Delta W_k^{(2)}-\tfrac{\sigma}{2}\tA_k^{(1,2)}\\
    \hat{V}_{k+1}&=(\Tilde{V}_{k+1}^2-\xi)e^{-\frac{\kappa h}{2}}+\xi\\
    \hat{U}_{k+1}&=\Tilde{U}_{k+1}+\frac{1}{2\kappa}(\Tilde{V}_{k+1}^2-\xi)(e^{-\frac{\kappa h}{2}}-1)+\tfrac{h}{2}(r-\tfrac{\xi}{2}),
\end{align*}
where $\xi\coloneqq \theta-\tfrac{\sigma^2}{4\kappa}$ and $\tA$ denotes the fake L\'evy area.\smallbreak

\end{document}